\documentclass{article}
\usepackage{microtype}
\usepackage{graphicx}
\usepackage{subfigure}
\usepackage{booktabs}
\usepackage{floatrow}
\newfloatcommand{capbtabbox}{table}[][\FBwidth]

\usepackage{adjustbox}
\usepackage[dvipsnames]{xcolor}
\definecolor{deepsetsblue}{RGB}{70,26,201}
\definecolor{discontinuousunitred}{RGB}{233,27,106}
\definecolor{subpatternsseablue}{RGB}{114,155,115}
\usepackage{algorithm}
\RequirePackage{amsmath}
\usepackage{amsthm,esint}
\usepackage[numbers]{natbib}
\RequirePackage{natbib}
\usepackage{bbm}
\usepackage{booktabs}
\usepackage{eso-pic,graphicx}
\usepackage[titletoc,title]{appendix}
\usepackage[nottoc]{tocbibind}
\usepackage{xparse}
\usepackage{appendix}
\usepackage{chngcntr}
\usepackage{etoolbox}
\usepackage{tikz}
\pgfdeclarelayer{edgelayer}
\pgfdeclarelayer{nodelayer}
\pgfsetlayers{edgelayer,nodelayer,main}  
\tikzstyle{Partioner Unit}=[fill={rgb,255: red,148; green,155; blue,129}, draw={rgb,255: red,67; green,224; blue,177}, shape=circle]
\tikzstyle{Partitioner Unit Blue}=[fill={rgb,255: red,153; green,194; blue,255}, draw={rgb,255: red,70; green,26; blue,201}, shape=circle]
\tikzstyle{Discontinuous Unit}=[fill={rgb,255: red,255; green,166; blue,166}, draw={rgb,255: red,233; green,27; blue,106}, shape=rectangle,  minimum width= 2.5em, minimum height= 2.5em]
\tikzstyle{Discontinuous Unit small}=[fill={rgb,255: red,255; green,166; blue,166}, draw={rgb,255: red,233; green,27; blue,106}, shape=rectangle]
\tikzstyle{dot}=[fill={rgb,255: red,179; green,179; blue,179}, draw=black, shape=circle]
\tikzstyle{Plus}=[fill={rgb,255: red,230; green,230; blue,230}, draw={rgb,255: red,114; green,114; blue,114}, shape=circle]
\tikzstyle{Predictor Arrow}=[fill={rgb,255: red,114; green,155; blue,115}, ->, draw={rgb,255: red,39; green,219; blue,165}]
\tikzstyle{Partitioner Arrow}=[fill={rgb,255: red,62; green,20; blue,229}, draw={rgb,255: red,60; green,31; blue,225}, ->]
\tikzstyle{Discontinuous Unit Information Flow}=[fill={rgb,255: red,189; green,14; blue,166}, draw={rgb,255: red,212; green,20; blue,136}, ->]
\tikzstyle{Plus}=[-, draw={rgb,255: red,133; green,133; blue,133}]
\usepackage{graphicx}
\usepackage{sidecap}
\usepackage{float}
\usepackage{enumerate,mdwlist}
\usepackage{epsfig}
\usepackage[bbgreekl]{mathbbol}
\usepackage{amsfonts,amsmath,enumerate}
\usepackage{amscd}
\usepackage{amsthm}
\usepackage{mleftright}
\DeclareSymbolFontAlphabet{\mathbb}{AMSb}
\DeclareSymbolFontAlphabet{\mathbbl}{bbold}
\usepackage{colortbl}

\usepackage{amsmath,amsfonts,amssymb,mathrsfs}\usepackage{bm}
\usepackage{latexsym}
\usepackage{mleftright}
\usepackage{amssymb}
\usepackage{color}
\usepackage{xifthen}
\usepackage[titletoc]{appendix}
\usepackage{footnote}
\usepackage{hyperref} 
\hypersetup{
	colorlinks = true,
	linkcolor = blue,
	anchorcolor = blue,
	citecolor = blue,
	filecolor = blue,
	urlcolor = blue
}
\usepackage{makeidx}
\usepackage{amsmath}
\usepackage{amsthm}
\usepackage{amsfonts}
\usepackage{amssymb}
\usepackage{mathrsfs}
\usepackage{graphicx}
\usepackage{amssymb,amsmath,amsthm,nccmath}
\usepackage{mathtools}%
\usepackage{enumerate}
\usepackage{dsfont}\usepackage{accents}\usepackage{bigstrut}
\usepackage{multirow}
\usepackage{scalerel}%
\usepackage{footnote}
\usepackage{etoolbox}\usepackage[utf8]{inputenc} %
\usepackage[T1]{fontenc}    
\usepackage{url}            %
\usepackage{booktabs}       %
\usepackage{amsfonts}       %
\usepackage{nicefrac}       %
\usepackage{microtype}      %
\usepackage{inconsolata}
\usepackage{mathptmx}\usepackage{hyperref}       %
\usepackage{url}            %
\usepackage{booktabs}       %
\usepackage{amsfonts}       %
\usepackage{nicefrac}       %
\usepackage{microtype}      %
\usepackage{amssymb,amsmath,amsthm,nccmath}
\usepackage{mathtools}%
\usepackage{graphicx}
\usepackage{enumerate}
\usepackage{dsfont}\usepackage{accents}\usepackage{bigstrut}
\usepackage{multirow}
\usepackage{scalerel}
\usepackage{hyphenat}

\newcommand{\rr}{{\mathbb{R}}}

\newcommand{\pp}{{\mathbb{P}}}

\newcommand{\xx}{{\mathbb{X}}}

\newcommand{\nn}{{\mathbb{N}}}

\NewDocumentCommand{\comp}{o}{
\text{Comp}\IfValueT{#1}{\left({#1}\right)}
}

\newcommand{\kkk}{X}
\newcommand{\fff}{{\NN}}
\renewcommand{\ggg}{{\NN[d,1]}}

\newcommand{\rrr}{R}

\newcommand{\rrflex}[1]{{\ensuremath{\rr^{#1}
}}}
\newcommand{\rrD}{{\rrflex{D}}}
\newcommand{\rrd}{{\rrflex{d}}}

\newcommand{\xxx}{X}

\NewDocumentCommand\argmin{o}{{\operatorname{argmin}\IfValueT{#1}{_{{#1}}}}}
\NewDocumentCommand\AF{o}{\operatorname{AF}\IfValueT{#1}{
		{
			\left({#1}\right)
		}
}}

\NewDocumentCommand\NNtrunc{o}{{
		\left \lceil{
			\mathcal{NN}^{\sigma}
			\IfValueT{#1}{_{#1}}
		}\right \rceil 
}}

\NewDocumentCommand\NNshal{oo}{
	{nn%
		_{\IfValueT{#1}{#1}}	
		\IfValueF{#2}{^{\fff}}\IfValueT{#2}{^{#2}}
	}
}
\NewDocumentCommand\NNho{oo}{
	{\mathcal{HNN}%
		_{R\IfValueT{#1}{#1}}	
		\IfValueF{#2}{^{\fff}}\IfValueT{#2}{^{#2}}
	}
}
\NewDocumentCommand\NNaff{oo}{
	{\mathcal{NN}%
		_{R,\IfValueT{#1}{#1}}	
		\IfValueF{#2}{^{a}}
	}
}
\NewDocumentCommand{\prodd}{oo}{
	\overset{{#2}}{
		\underset{{#1}}{
			\circlearrowleft
		}
	}
}

\NewDocumentCommand{\NN}{oo}{
	{\operatorname{NN}\IfValueT{#2}{^{#2}}\IfValueF{#2}{^{\sigma}}
		\IfValueF{#1}{_{d,D}}
		\IfValueT{#1}{_{{{{#1}}}}}
	}
}
\NewDocumentCommand{\intt}{o}{{\operatorname{int}
		\IfValueT{#1}{\left({#1}\right)}
}}
\NewDocumentCommand\rsupp{mo}{{\operatorname{supp}
		\IfValueT{#1}{\left({#1}
			\middle|
			\IfValueT{#2}{{{#2}}}\IfValueF{#2}{
				{
					K_{\cdot}
				}
			}
			\right)}
}}
\newtheorem{defn}{Definition}[section]

\newtheorem{ass}{Assumption}[defn]

\newtheorem{prop}{Proposition}[defn]

\newtheorem{lem}{Lemma}[defn]
\newtheorem{ex}{Example}[defn]
\newtheorem{thrm}{Theorem}[defn]
\newtheorem{rremark}{Remark}[defn]

\newtheorem*{ass*}{Assumption}
\newtheorem*{thrm*}{Theorem}
\newtheorem*{cor*}{Corollary}
\newtheorem*{prop*}{Proposition}
\newtheorem*{lem*}{Lemma}
\newtheorem{rremark*}{Remark}
\NewDocumentCommand{\PCNN}{mo}{
	{
		\operatorname{PC}{#1}
		\IfValueT{#2}{{_{#2}}}
	}
}
\newcommand{\aname}{{\mbox{PCNN}}}
\newcommand{\anames}{{\mbox{PCNNs}}}
\usepackage{xfrac}
\usepackage{longtable}
\usepackage{marginnote}
\usepackage{algorithmic}
\usepackage{floatrow}
\usepackage{blindtext}
\usepackage[letterpaper,left=2.54cm,top=2.54cm,bottom=2.54cm,right=2.54cm]{geometry}
\usepackage{fancyhdr}
\title{Learning Sub-Patterns in Piecewise Continuous Functions}
\author{Anastasis Kratsios
	\thanks{
	University of Basel, Department of Mathematics and Informatics, Spiegelgasse 1, 4051 Basel email: \textit{anastasis.kratsios@math.ethz.ch}
	}%
	\and
	Behnoosh Zamanlooy
	\thanks{
		Department of Informatics, Computation, and Economics, University of Z\"{u}rich, Binzm\"{u}hlestrasse 14, 8050 Z\"{u}rich.
		email: \textit{bzamanlooy@ifi.uzh.ch}
	}%
}

\begin{document}
\maketitle
\begin{abstract}
Most stochastic gradient descent algorithms can optimize neural networks that are sub-differentiable in their parameters; however, this implies that the neural network's activation function must exhibit a degree of continuity which limits the neural network model's uniform approximation capacity to continuous functions.  This paper focuses on the case where the discontinuities arise from distinct sub-patterns, each defined on different parts of the input space.   We propose a new discontinuous deep neural network model trainable via a decoupled two-step procedure that avoids passing gradient updates through the network's only and strategically placed, discontinuous unit.  We provide approximation guarantees for our architecture in the space of bounded continuous functions and universal approximation guarantees in the space of piecewise continuous functions which we introduced herein.  We present a novel semi-supervised two-step training procedure for our discontinuous deep learning model, tailored to its structure, and we provide theoretical support for its effectiveness.  The performance of our model and trained with the propose procedure is evaluated experimentally on both real-world financial datasets and synthetic datasets.
\end{abstract}
\noindent
\textbf{Keywords:} Piecewise Continuous Functions, Universal Approximation, Discontinuous Feedforward Networks, Deep Zero-Sets, Set-Valued Universal Approximation, Geometric Deep Learning, Portfolio Replication.\hfill\\
\textbf{MSC:} Artificial neural networks and deep learning (68T07),  Set-valued and variational analysis (49J53), Partitions of sets (05A18), Parallel numerical computation (65Y05), Randomized algorithms (68W20), Financial Markets (91G15).  
\let\thefootnote\relax\footnotetext{This research was supported by the ETH Z\"{u}rich Foundation and by the ERC.}
\section{Introduction}
Since their introduction in \cite{mcculloch1943logical}, neural networks have led to numerous advances across various scientific areas.  These include, mathematical finance in \cite{cont2010stochastic,buehler2019deep,cuchiero2020generative}, computer vision and neuroimaging in \cite{simonyan2014very,moore2019using}, signal processing in~\cite{lapedes1987nonlinear,krishnan2015deep}, and climate change modeling in~\cite{climatechnage}.  From the theoretical vantage point, these methods' success lies in the harmony between their expressivity \citep{hornik1990universal,barron1993universal,poggio2017and,Yaroski2020Smooth}, the training algorithms which can efficiently leverage this expressibility \citep{ADAM2015KingmaB14,JMLRmaximumprinciple,JMLRautodiffSurvey,patrascu2021stochastic} and, the implicit inductive bias of deep neural model trained with these methods \cite{NEURIPS2019_c4ef9c39,heiss2019implicit}.  

This paper probes the first two aspects when faced with the task of learning piecewise continuous functions.  We first identify approximation-theoretic limitations to commonly deployed feedforward neural networks (FFNNs); i.e.: with continuous activation functions, and then fill this gap with a new deep neural model ($\anames$) together with a randomized and parallelizable training meta-algorithm that exploits the $\aname$'s structure.  

The description of the problem, and our results, begins by revisiting the classical \textit{universal approximation theorems} \citep{Cybenko,funahashi1989approximate,hornik1991approximation}.  
Briefly, these classical universal approximation results state that, when a phenomenon is governed by some continuous target function $f$, then FFNNs with continuous function $\sigma$ can control the worst-case approximation of error to arbitrary precision.  If $f$ is discontinuous, as is for instance the case in many signal processing or mathematical finance \citep{lee2008jumps,tankov2003financial} situations then, the uniform limit theorem from classical topology \cite{munkres2014topology} guarantees that the worst-case approximation error of $f$ by FFNNs cannot be controlled; nevertheless, the average error incurred by approximating $f$ by FFNNs can be \citep{barron1993universal,gribonval2021approximation,SIEGEL2020313}.  Essentially, this means that if $f$ is discontinuous then when approximating it by an FFNN there must be a ``small'' portion of the (non-empty) input space $X\subseteq \rr^d$ whereon the approximation can become arbitrarily poor.  

If $f:X\rightarrow \rr^D$ is a \textit{piecewise continuous function}, by which we mean that it can be represented as:
\begin{equation}
	f = \sum_{n=1}^N I_{K_n}f_n
	;
	\label{eq_sub_patterns}
\end{equation}
for some integer $N$, continuous functions $f_n:\rr^d\rightarrow \rr^D$ and, some (non-empty) compact subsets $K_n\subseteq X$ then, the regions where the approximation of $f$ by FFNN is poor corresponds precisely to the regions where the \textit{parts} $\{K_n\}_{n=1}^N$ interface (whenever the \textit{subpatterns} $\{f_n\}_{n=1}^N$ miss-match thereon).  

In principle, the existence result of \cite{kratsios2019UATs} implies that guarantees for uniform approximation should be possible with a deep neural model and one such likely candidate are deep feedforward with discontinuous activation functions.  However, these types of networks are not compatible with most commonly used (stochastic) gradient descent-type algorithms.  
A few methods for training such models are available.  For instance, \cite{findlay1989} describes a heuristic approach, but its empirical performance and theoretical guarantees are not explored.  In \cite{ferreira2005solving} the author proposes a linear programming approach to training shallow feedforward networks with threshold activation function whose hidden weights and biases are fixed.  This avoids back-propagating through the non-differentiable threshold activation function; however, the method is specific to their shallow architecture.  Similarly, \cite{huang2006can} considers an extreme learning machine approach by randomizing all but the network's final linear layer, which reduces the training task to a classical linear regression problem.  However, this approach's provided approximation results are strictly weaker than the known guarantees for classical feedforward networks with a continuous activation function obtained, as are derived for example in \cite{kidger2020universal}.  

Our proposed solution to overcoming the problem that: worst-case \textit{universal approximation} by FFNNs is limited to continuous functions, begins by acknowledging that approximating an arbitrary discontinuous (or even integrable) function $f$ is an \textit{unstructured approximation} problem; whereas approximating a piecewise continuous function is a \textit{structured approximation problem}; i.e. there is additional structure and it should be encoded into the machine learning model in order to achieve competent predictive performance.  Typical examples of structured approximation problems include assimilating graph structure into the learning model \citep{ZHOU202057,2021arXiv210413478B}, manifold-valued neural networks \cite{ganea2018hyperbolic,kratsios2020non,kratsios2021quantitative,ZamanlooyConstraints}, encoded symmetries into the trained model \citep{pmlrv97cohen19d,petersen2020equivalence,yarotsky2021universal}, or encode inevitability \cite{NEURIPS2019_7ac71d43,NEURIPS2018_69386f6b,grathwohl2018scalable,kratsios2021neu}.  Here, the relevant structure is given by the (continuous) subpatterns $\{f_n\}_{n=1}^N$ and (non-empty compact) parts $\{K_n\}_{n=1}^N$ representing the function $f$ in~\eqref{eq_sub_patterns}.  

\begin{figure}[ht]
\centering
\begin{floatrow}
\ffigbox{
\begin{adjustbox}{width=\columnwidth,center}
\centerline{\includegraphics[scale=0.33]{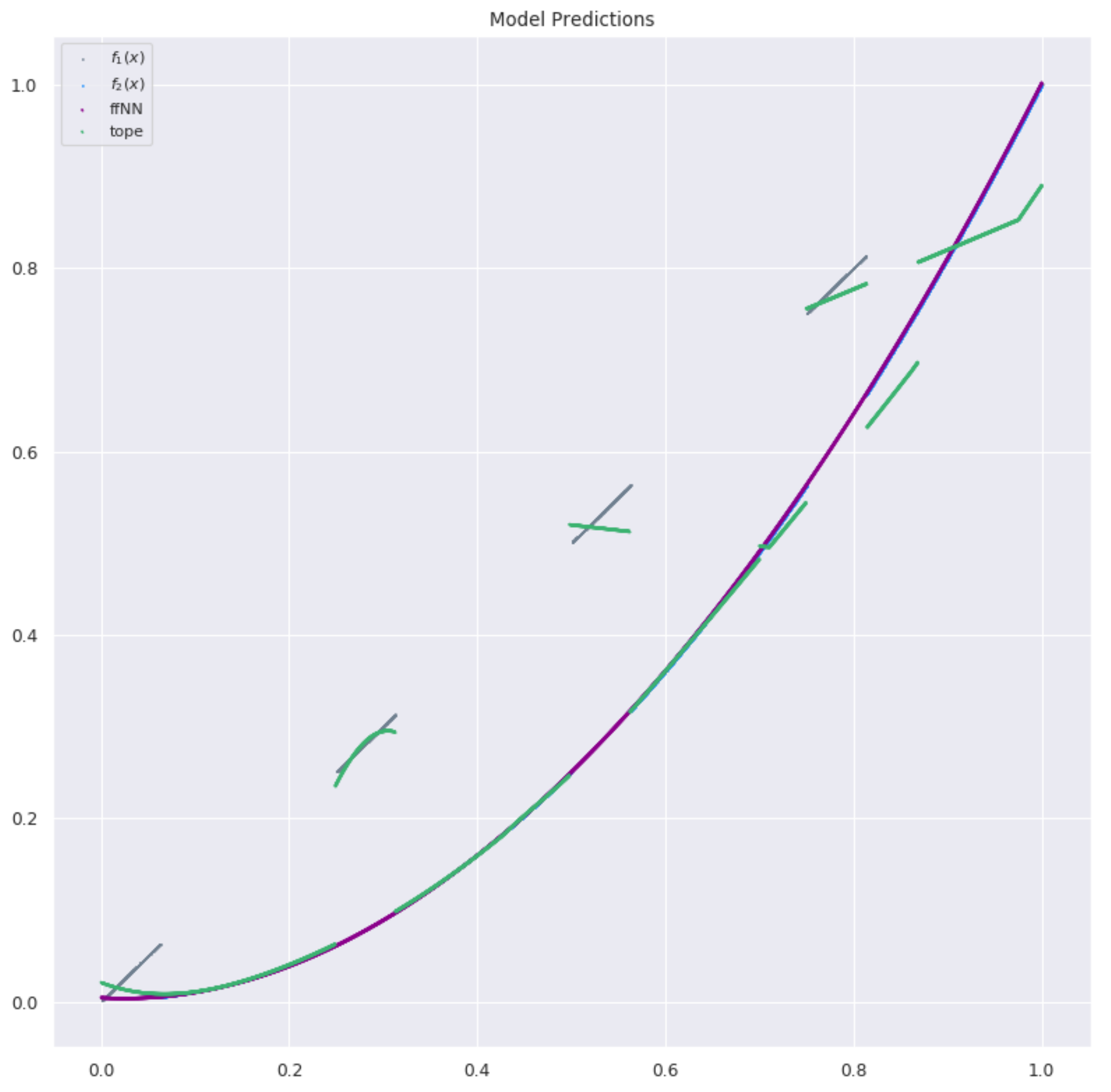}}
\end{adjustbox}
}{%
  \caption{Approximation of sub-patterns by FFNN.}
		\label{fig_Motivational_Demonstration_of_PCNNs}
}
\ffigbox{%
\begin{adjustbox}{width=\columnwidth,center}
\scalebox{0.2}{
\begin{tikzpicture}
	\begin{pgfonlayer}{nodelayer}
		\node [style=Partitioner Unit Blue] (0) at (0, 0) {};
		\node [style=Partitioner Unit Blue] (1) at (0, -2) {};
		\node [style=Partitioner Unit Blue] (2) at (0, 2) {};
		\node [style=Partitioner Unit Blue] (3) at (2.5, 0.5) {};
		\node [style=Partitioner Unit Blue] (4) at (2.5, 1.5) {};
		\node [style=Partitioner Unit Blue] (5) at (2.5, -0.5) {};
		\node [style=Partitioner Unit Blue] (6) at (2.5, -1.5) {};
		\node [style=Partitioner Unit Blue] (7) at (4.75, 0.75) {};
		\node [style=Partitioner Unit Blue] (9) at (4.75, -0.5) {};
		\node [style=dot] (15) at (0, 2) {};
		\node [style=dot] (16) at (0, 0) {$x_2$};
		\node [style=dot] (17) at (0, -2) {};
		\node [] (18) at (0, 2) {};
		\node [] (19) at (0, 2) {};
		\node [] (20) at (0, 2) {};
		\node [] (21) at (0, 2) {};
		\node [style=dot] (22) at (0, 2) {$x_1$};
		\node [style=dot] (23) at (0, -2) {$x_3$};
		\node [style=Partioner Unit] (27) at (2.5, 5.5) {};
		\node [style=Partioner Unit] (28) at (2.5, 4.5) {};
		\node [style=Partioner Unit] (29) at (2.5, 3.5) {};
		\node [style=Partioner Unit] (30) at (4.5, 6.5) {};
		\node [style=Partioner Unit] (31) at (4.5, 5.5) {};
		\node [style=Partioner Unit] (32) at (4.5, 4.5) {};
		\node [style=Partioner Unit] (33) at (4.5, 3.5) {};
		\node [style=Partioner Unit] (34) at (2.5, -3.25) {};
		\node [style=Partioner Unit] (35) at (2.5, -4.5) {};
		\node [style=Partioner Unit] (36) at (4.5, -2.5) {};
		\node [style=Partioner Unit] (37) at (4.5, -4.5) {};
		\node [style=Partioner Unit] (38) at (4.5, -5.5) {};
		\node [style=Partioner Unit] (39) at (4.5, -3.5) {};
		\node [style=Partioner Unit] (40) at (4.5, -6.5) {};
		\node [style=Partioner Unit] (41) at (6.5, 3.5) {};
		\node [style=Partioner Unit] (42) at (6.5, 4.5) {};
		\node [style=Partioner Unit] (43) at (6.5, 5.5) {};
		\node [style=Partioner Unit] (44) at (6.5, -2.5) {};
		\node [style=Partioner Unit] (45) at (6.5, -3.5) {};
		\node [style=Partioner Unit] (46) at (6.5, -4.5) {};
		\node [style=Partioner Unit] (47) at (8.25, 5.5) {};
		\node [style=Partioner Unit] (48) at (8.25, 4.5) {};
		\node [style=Partioner Unit] (49) at (8.25, 3.5) {};
		\node [style=Partioner Unit] (50) at (10, 5.5) {};
		\node [style=Partioner Unit] (51) at (10, 4.5) {};
		\node [style=Partioner Unit] (52) at (10, 3.5) {};
		\node [style=Partioner Unit] (53) at (6.5, 5.5) {};
		\node [style=Partioner Unit] (54) at (8.25, 5.5) {};
		\node [style=dot] (59) at (15, 3.75) {$\hat{f}(x)_1$};
		\node [style=dot] (60) at (15, 1.25) {$\hat{f}(x)_2$};
		\node [style=dot] (61) at (15, -1.25) {$\hat{f}(x)_3$};
		\node [style=dot] (62) at (15, -3.75) {$\hat{f}(x)_3$};
		\node [] (64) at (10, 5.5) {};
		\node [] (65) at (10, 4.5) {};
		\node [] (66) at (10, 3.5) {};
		\node [style=Discontinuous Unit] (72) at (11.5, 0) {};
	\end{pgfonlayer}
	\begin{pgfonlayer}{edgelayer}
		\draw [style=Partitioner Arrow] (2) to (3);
		\draw [style=Partitioner Arrow] (2) to (4);
		\draw [style=Partitioner Arrow] (2) to (5);
		\draw [style=Partitioner Arrow] (2) to (6);
		\draw [style=Partitioner Arrow] (0) to (3);
		\draw [style=Partitioner Arrow] (0) to (4);
		\draw [style=Partitioner Arrow] (0) to (5);
		\draw [style=Partitioner Arrow] (0) to (6);
		\draw [style=Partitioner Arrow] (1) to (6);
		\draw [style=Partitioner Arrow] (1) to (5);
		\draw [style=Partitioner Arrow] (1) to (3);
		\draw [style=Partitioner Arrow] (1) to (4);
		\draw [style=Partitioner Arrow] (3) to (7);
		\draw [style=Partitioner Arrow] (4) to (7);
		\draw [style=Partitioner Arrow] (3) to (9);
		\draw [style=Partitioner Arrow] (4) to (9);
		\draw [style=Partitioner Arrow] (5) to (9);
		\draw [style=Partitioner Arrow] (5) to (7);
		\draw [style=Partitioner Arrow] (6) to (9);
		\draw [style=Partitioner Arrow] (6) to (7);
		\draw [style=Predictor Arrow] (22.center) to (29);
		\draw [style=Predictor Arrow] (22.center) to (28);
		\draw [style=Predictor Arrow] (22.center) to (27);
		\draw [style=Predictor Arrow] (22.center) to (34);
		\draw [style=Predictor Arrow] (22.center) to (35);
		\draw [style=Predictor Arrow] (16) to (34);
		\draw [style=Predictor Arrow] (16) to (35);
		\draw [style=Predictor Arrow] (16) to (29);
		\draw [style=Predictor Arrow] (16) to (28);
		\draw [style=Predictor Arrow] (16) to (27);
		\draw [style=Predictor Arrow] (23.center) to (35);
		\draw [style=Predictor Arrow] (23.center) to (34);
		\draw [style=Predictor Arrow] (23.center) to (28);
		\draw [style=Predictor Arrow] (23.center) to (27);
		\draw [style=Predictor Arrow] (29) to (33);
		\draw [style=Predictor Arrow] (29) to (32);
		\draw [style=Predictor Arrow] (29) to (31);
		\draw [style=Predictor Arrow] (29) to (30);
		\draw [style=Predictor Arrow] (34) to (36);
		\draw [style=Predictor Arrow] (34) to (39);
		\draw [style=Predictor Arrow] (34) to (37);
		\draw [style=Predictor Arrow] (34) to (38);
		\draw [style=Predictor Arrow] (34) to (40);
		\draw [style=Predictor Arrow] (35) to (40);
		\draw [style=Predictor Arrow] (35) to (38);
		\draw [style=Predictor Arrow] (35) to (37);
		\draw [style=Predictor Arrow] (35) to (39);
		\draw [style=Predictor Arrow] (35) to (36);
		\draw [style=Predictor Arrow] (28) to (33);
		\draw [style=Predictor Arrow] (28) to (32);
		\draw [style=Predictor Arrow] (28) to (31);
		\draw [style=Predictor Arrow] (28) to (30);
		\draw [style=Predictor Arrow] (27) to (30);
		\draw [style=Predictor Arrow] (27) to (31);
		\draw [style=Predictor Arrow] (27) to (32);
		\draw [style=Predictor Arrow] (27) to (33);
		\draw [style=Predictor Arrow] (36) to (44);
		\draw [style=Predictor Arrow] (39) to (44);
		\draw [style=Predictor Arrow] (37) to (44);
		\draw [style=Predictor Arrow] (38) to (44);
		\draw [style=Predictor Arrow] (40) to (44);
		\draw [style=Predictor Arrow] (36) to (45);
		\draw [style=Predictor Arrow] (39) to (45);
		\draw [style=Predictor Arrow] (39) to (46);
		\draw [style=Predictor Arrow] (38) to (45);
		\draw [style=Predictor Arrow] (40) to (45);
		\draw [style=Predictor Arrow] (38) to (46);
		\draw [style=Predictor Arrow] (37) to (45);
		\draw [style=Predictor Arrow] (33) to (41);
		\draw [style=Predictor Arrow] (33) to (42);
		\draw [style=Predictor Arrow] (33) to (43);
		\draw [style=Predictor Arrow] (32) to (41);
		\draw [style=Predictor Arrow] (32) to (42);
		\draw [style=Predictor Arrow] (32) to (43);
		\draw [style=Predictor Arrow] (31) to (43);
		\draw [style=Predictor Arrow] (31) to (42);
		\draw [style=Predictor Arrow] (31) to (41);
		\draw [style=Predictor Arrow] (30) to (41);
		\draw [style=Predictor Arrow] (30) to (42);
		\draw [style=Predictor Arrow] (30) to (43);
		\draw [style=Predictor Arrow] (53) to (54);
		\draw [style=Predictor Arrow] (53) to (48);
		\draw [style=Predictor Arrow] (53) to (49);
		\draw [style=Predictor Arrow] (42) to (54);
		\draw [style=Predictor Arrow] (42) to (48);
		\draw [style=Predictor Arrow] (42) to (49);
		\draw [style=Predictor Arrow] (41) to (54);
		\draw [style=Predictor Arrow] (41) to (48);
		\draw [style=Predictor Arrow] (41) to (49);
		\draw [style=Predictor Arrow, in=180, out=0] (49) to (52);
		\draw [style=Predictor Arrow] (49) to (51);
		\draw [style=Predictor Arrow] (49) to (50);
		\draw [style=Predictor Arrow] (48) to (52);
		\draw [style=Predictor Arrow] (48) to (51);
		\draw [style=Predictor Arrow] (48) to (50);
		\draw [style=Predictor Arrow] (54) to (51);
		\draw [style=Predictor Arrow] (54) to (50);
		\draw [style=Predictor Arrow] (54) to (52);
		\draw [style=Discontinuous Unit Information Flow] (64.center) to (72);
		\draw [style=Discontinuous Unit Information Flow] (65.center) to (72);
		\draw [style=Discontinuous Unit Information Flow] (66.center) to (72);
		\draw [style=Discontinuous Unit Information Flow] (72) to (59);
		\draw [style=Discontinuous Unit Information Flow] (72) to (60);
		\draw [style=Discontinuous Unit Information Flow] (72) to (61);
		\draw [style=Discontinuous Unit Information Flow] (72) to (62);
		\draw [style=Discontinuous Unit Information Flow] (7) to (72);
		\draw [style=Discontinuous Unit Information Flow] (9) to (72);
		\draw [style=Discontinuous Unit Information Flow] (46) to (72);
		\draw [style=Discontinuous Unit Information Flow] (45) to (72);
		\draw [style=Discontinuous Unit Information Flow] (44) to (72);
		\matrix [draw,below left] at (current bounding box.north east) {
  \node [Partioner Unit,label=right: Sub Patterns] {}; \\
  \node [Partitioner Unit Blue,label=right: Deep Classifier] {}; \\
  \node [Discontinuous Unit small,label=right: Discontinuous Unit] {}; \\
};
	\end{pgfonlayer}
\end{tikzpicture}
}
\end{adjustbox}
}{%
    \caption{{A PCNN $\hat{f}$ from $\rr^3$ to $\rr^4$ with two sub-patterns.}}
\label{fig_PCNNs}
}
\end{floatrow}
\end{figure}

Our proposed deep neural models, the $\anames$, reflects this structure by approximately parameterized the sub-patterns $\{f_n\}_{n=1}^N$ by in dependant FFNNs, approximately parameterizing the parts $\{K_n\}_{n=1}^N$ by zero-sets of FFNNs, and the combining these independent parts via a single discontinuous unit defined shortly.  
An instance of our architecture is illustrated graphically in Figure~\ref{fig_PCNNs}.  Since each FFNN component of our architecture is in dependant from one another and are only regrouped at their final outputs by the discontinuous unit, illustrated in red in Figure~\ref{fig_PCNNs}, then we can leverage this structure to decouple the training of each component of our $\aname$ model and then re-combine them when producing predictions.  Thus, our decoupled training procedure allows effectively avoids passing any gradients through the discontinuous unit.  Thus, our architecture enjoys the expressivity of a discontinuous units (guaranteed by our universal approximation theorems for piecewise continuous functions in Theorems~\ref{thrm_improvement},~\ref{thrm_improvement_largeness}, and~\ref{thrm_UAT_PC}) while being pragmatically trainable (illustrated in by our numerical experiments and previewed in Figure~\ref{fig_Motivational_Demonstration_of_PCNNs}).

Figure~\ref{fig_Motivational_Demonstration_of_PCNNs} provides a concrete visual motivation of our results and approach.  It illustrates the challenge of learning a piecewise continuous functions with two parts (in grey and orange) by an FFNN with ReLU activation function $2$ hidden layers and $100$ neurons in each layer (in purple) and a comparable $\aname$ model (green).  After 3000 iterations of ADAM algorithm of \cite{ADAM2015KingmaB14}, the training stabilizes at a mean absolute error (MAE) of about $0.1477$.  In contrast the $\anames$ are capable produced an MAE of $0.986$.

\subsubsection*{Organization of Paper}
Section~\ref{s_Background} contains the background material required in the formulation of the paper's main results.  Section~\ref{s_Main} includes the paper's main results.  These include qualitative universal approximation guarantees for the $\aname$ as a whole, as well as, for each of its individual components, quantitative approximation guarantees for the $\aname$ given a partition, as well as a result showing that the PCNN are dense in a significantly larger space than the FFNNs are, for the uniform distance.  We then introduce a randomized algorithm that exploits the $\aname$'s structure to train it.  Theoretical guarantees surrounding this algorithm are also proven.  Lastly, in Section~\ref{s_Implementation} we validate our theoretical claims by training the $\aname$ architecture using our proposed meta-algorithm to generate predictions from various real-world financial and synthetic datasets.  The model's performance is benchmarked against comparable deep neural models trained using conventional training algorithms.
\section{Preliminaries}\label{s_Background} 
Let us fix some notation.  
The following notation is used and maintained throughout the paper.  The space of continuous functions from $X$ to $ \rrD$ is denoted by $C(\rrd, \rrD)$. When $D = 1$, we follow the convention of denoting $C(\rrd, \rr)$ by $C(\rrd)$. Throughout this paper, we fix a continuous \textit{activation function} $\sigma:\rr\rightarrow \rr$.  We denote the set of all feedforward networks from $\rrd$ to $\rrD$ by $\NN$ and we denote $\sigma_{\operatorname{sigmoid}}(x)\triangleq \frac{e^x}{1+e^x}$.

We denote the set of positive integers by $\nn^+$.  The cardinality (or size) of a set $A$ is denoted $\# A$.  
In this manuscript, $d,D$ will always be non-negative integers.  Moreover, $X$ will always denote a non-empty compact subset of $\rrd$.
\subsection{The PCNN Model}\label{s_Background_ss_PCNN_definition}
Before reviewing the relevant background for the formulation of our results, we first define the $\aname$ architecture.  
\begin{defn}[$\aname$]\label{defn_PCNNs}
	$\aname$ is a function $\hat{f}:X\rightarrow \rrD$ with representation
	$$
	\hat{f}=\sum_{n=1}^N 
	{\color{subpatternsseablue}{
	    \hat{f}_n(x)
	}}
	I_{
	    {\color{deepsetsblue}{
	        \hat{K}_n
	    }}
	}(x),
	$$
	where the partition $\left\{\hat{K}_n\right\}_{n=1}^N$ is given by the ``deep zero-sets'' defined as:
	\begin{equation}
	\hat{K}_n
	=
	\left\{x \in X: \,
	1-
	{\color{discontinuousunitred}{
	    I_{(\gamma,1]}\circ \sigma_{\operatorname{sigmoid}}
	 }}
	 (
	{\color{subpatternsseablue}{
	    \hat{c}(x)_n
	}}
	)=0\right\}
	\label{eq_deep_zero_sets_definition}
	,
	\end{equation}
	and where the ``sub-patterns'' $\hat{f}_1,\dots,\hat{f}_N$ are FFNNS in $\NN[d,D]$, $
    \hat{c}\in \NN[d,N]
	$, and $0<\gamma\leq 1$, $N\in \nn_+$.  The set of $\aname$s is denoted by $\PCNN{\NN}$.  
	Each $\hat{f}_1,\dots,\hat{f}_N$ is called a sub-pattern of the $\aname$ $\hat{f}$.  
\end{defn} 

An instance of the $\aname$ architecture is illustrated in Figure~\ref{fig_PCNNs}.  
The illustration depicts an architecture from a $d=3$ dimensional input space to a $D=4$ dimensional output space, with $2$ ``deep zero-sets'' $(\hat{K}_n)_{n=1}^2$, and, accordingly, two sub-patterns {\color{subpatternsseablue}{$\hat{f}_1$}} and {\color{subpatternsseablue}{$\hat{f}_2$}} defining the respective sub-pattern on their respective deep zero-sets.  
The deep zero-sets are built by feeding the {\color{deepsetsblue}{deep classifier $\hat{c}$}} into the {\color{discontinuousunitred}{discontinuous unit $x\mapsto x\cdot I_{(\gamma,1]}\circ \sigma_{\mbox{sigmoid}}$}}.  
Figure~\ref{fig_PCNNs} highlights that each trainable part of the network, namely the sub-patterns {\color{subpatternsseablue}{$\hat{f}_1$}} and {\color{subpatternsseablue}{$\hat{f}_2$}} and the {\color{deepsetsblue}{deep classifier $\hat{c}$}}, all process any input data \textit{independantly} and therefore can be parallelized.  
The outputs $\hat{y}^{(1)}\triangleq \hat{f}(x_1,x_2,x_3)$, $\hat{y}^{(2)}\triangleq \hat{f}(x_1,x_2,x_3),$ and $\hat{C}\triangleq \hat{c}(x_1,x_2,x_3)$ of each parallelizable sub-patterns then fed into the {\color{discontinuousunitred}{discontinuous unit}}: 
\[
(\hat{y}^{(1)},\hat{y}^{(2)},\hat{C})\mapsto 
(\hat{y}^{(1)} I_{(\gamma,1]}\circ \sigma_{\mbox{sigmoid}}(\hat{C}_1) +
(\hat{y}^{(2)} I_{(\gamma,1]}\circ \sigma_{\mbox{sigmoid}}(\hat{C}_2)
,
\]
which simultaneously defines the deep zero-sets and decides which sub-pattern ({\color{subpatternsseablue}{$\hat{f}_i$}}) is to be activated.  Figure~\ref{fig_PCNNs} shows that the {\color{subpatternsseablue}{sub-patterns $\hat{f}_1$}} definition of $\aname$ need not have the configuration of hidden units.  Indeed when training a PCNN, described in Section~\ref{s_Implementation}, we will see that each of these {\color{subpatternsseablue}{sub-patterns}} is trainable in parallel to one another and therefore their widths and depths may be selected independently of one another.  

\subsection{Set-Valued Analysis}\label{sss_Set_Val_Analysis}
When performing binary classification, we want to learn which $x \in X$ belongs to a fixed subset $K\subseteq X$.  Thus, a classifier $\hat{c}:X\rightarrow \{0,1\}$ is typically trained to approximate $K$'s \textit{indicator function} $I_K$.  In this case, the results of \cite{farago1993strong} guarantee that $I_K$ can be approximated point-wise with high probability.  However, even if the approximation guarantees are strengthened and $\hat{c}$ (deterministically) approximates $I_K$ point-wise, the approximation of $K$  can be wrong. 
\begin{ex}\label{ex_difference_pointwise}
Let $X=[-1,1]$, $K=[0,1]$, and $\{K_k\}_{k=1}^{\infty}$ be the constant sequence of subsets of $X$ given by $K_k\triangleq (0,1)$.  Then, the pointwise limit of $I_{K_k}$ equals to $I_{(0,1)}$.  Thus, the points $\{0,1\}$ are always misclassified; even asymptotically.    
\end{ex}
The issue emphasised by Example~\eqref{ex_difference_pointwise} is that the points $\{0,1\}\subset K$ are limits of some sequence of points in the approximating sets $K_k$; for example, of the respective sequences $\{\frac1{k}\}_{k=1}^{\infty}$ and $\{1-\frac1{k}\}_{k=1}^{\infty}$.  Thus, a correct mode of convergence for sets should rather qualify a limiting set as containing all the limits of all sub-sequences $\{x_{k_j}\}_{j=1}^{\infty}$ where $x_k\in K_{k}$ for each $k\in\nn^+$.  This is implied, see \cite{AubinFrankowskaSetValuedAnalysis}, by the convergence according to the \textit{Hausdorff distance} $d_H$ defined for any pair of subsets $A,B\subseteq X$
via:
$
d_{H|X}(A,B)\triangleq \max\left\{\,\sup_{a \in A} \|a-B\|,\, \sup_{b \in B} \|A-b\| \,\right\},
$
where $\|a-B\|\triangleq \inf_{b\in B}\|a-b\|$ (similarly, we define $\|A-b\|\triangleq \|b-A\|$).  The \textit{Hausdorff distance} between $A$ and $B$ represents the largest distance which can be traversed if a fictive adversary assigns a starting point from either of $A$ or $B$.  Whenever it is clear what $X$ is, then we simply denote the Hausdorff distance $d_{H|X}$ on $X$ by $d_H$. As discussed in \cite{AubinFrankowskaSetValuedAnalysis}, the Hausdorff distance defines a metric on the the set $\operatorname{Comp}(X)$, whose elements are non-empty compact subsets of $X$ with Hausdorff distance equipped with the metric $d_H$.

\subsection{Partitions}\label{s_Background_ss_Partitions}
Fix $N\in \nn^+$ with $N>1$.  Throughout this paper, a \textit{partition} $\{K_n\}_{n=1}^{N}$ of the input space $\rrd$, means a collection of compact subsets of $\rrd$.  
We use the term "partition" loosely to mean any finite family $\{K_n\}_{n=1}^N$ of non-empty compact subsets of $X$.  Therefore, unless otherwise specified, we do not require that $\intt[K_n]\cap \intt[K_m]$ to be disjoint.  Thus, our results include that situation as a special case.  

\section{Main Results}\label{s_Main} 
Our main results are now presented.  All proofs are relegated to the paper's appendix.  
\subsection{Universal Approximation Guarantees}\label{ss_improvements}
In this section, we examine the asymptotic approximation capabilities of the $\aname$ model, and we compare and contrast it to that of FFNNs.  Two types of approximation theoretic results are considered.  The first examines the asymptotic approximation capabilities of $\PCNN{\fff}$ if the architecture's parameters were simultaneously optimizable.  However, as discussed in the paper's introduction, this cannot be done with most available (stochastic) gradient descent-type algorithms. Therefore this class of universal approximation results describes the "gold standard" of what $\PCNN{\fff}$ could be expected to approximate.   

However, since $\PCNN{\fff}$ are specifically designed to be trainable in a two-step decoupled procedure (via Meta-Algorithm~\ref{metaalgo_GET_PCNN} below) to avoid these non-differentiability issues.  The next class of universal approximation results are specifically designed to quantify the approximate capabilities of $\PCNN{\fff}$ when approximating piecewise continuous functions of the form~\eqref{eq_sub_patterns}.  These results are broken into two stages. First, we introduce a space of piecewise continuous functions on which a meaningful universal approximation theorem can be formulated, and we examine some of the properties of these new spaces.  Next, we show that $\PCNN{\fff}$ is universal in this space, and we quantify its approximation efficiency, analogously to \cite{pmlrv75yarotsky18a}.  Lastly, we show that, just as in the "gold standard case," FFNNs are not dense in this space; whenever we are approximating piecewise continuous functions with $N>1$ in~\eqref{eq_sub_patterns}.  

\subsubsection{Gold Standard: Maximum Approximation Capabilities}\label{sss_Approx_tope}
We find that $\PCNN{\fff}$ can approximate many more functions uniformly than $\fff$.  We do this by comparing the largest space in which $\PCNN{\fff}$ is universal, to the largest space in which $\fff$ is universal.  Since universality (or density) is an entirely topological property, then we compare these two spaces using purely topological criteria.  We observe the following.  
\begin{lem}\label{lem_boundedness_of_Architopes}
Every $f \in \PCNN{\fff}$ is bounded on $X$; i.e.: $\sup_{x \in X} \|f(x)\|<\infty$.  
\end{lem}
We need an ambient space to make our comparisons.  Thus, we begin by viewing $\PCNN{\fff}$ within the Banach space of all bounded functions from $X$ to $\rrD$, denoted by $\mathcal{B}(X,\rrD)$, equipped with the following norm:
$$
\|f\|_{\infty}\triangleq 
\sup_{x \in X}\left\|
f(x)
\right\|
.
$$
We use $\overline{\PCNN{\fff}}$ to denote the closure of $\PCNN{\fff}$ in $\mathcal{B}(X,\rrD)$; that is, $\overline{\PCNN{\fff}}$ denotes the collection of all $f \in \mathcal{B}(X,\rrD)$ for which there is a convergent sequence $\{f_k\}_{k=1}^{\infty}\subset \PCNN{\fff}$ with limit $f$.  
In other words, $\overline{\PCNN{\fff}}$ and $\overline{\fff}$ are the \textit{largest space of bounded functions} on $X$ in which $\PCNN{\fff}$ and $\fff$ are respectively universal, with respect to $\|\cdot\|_{\infty}$.  

Next, we show that $\overline{\fff}$ is relatively small in comparison to $\overline{\PCNN{\fff}}$.  To make this comparison, we appeal to an opposite concept to universality (density), i.e.: \textit{nowhere denseness}, which means that the only open subset of $\overline{\PCNN{\fff}}$ contained within $\overline{\fff}$ is the \textit{empty-set}.   
\begin{ex}\label{ex_concrete_description_dense_vs_nowheredense}
The set of linear models $\{ax+b:a,b\in \rr\}\subset C(\rr)$ is nowhere dense in $C(\rr)$.  In contrast, the polynomial models $\left\{
\sum_{n=0}^N a_n x^n:\, a_n\in \rr,\, N\in \nn
\right\}$ are dense in $C(\rr)$ (see \cite{WeirstrassConstructive2018Petrakis}).  
\end{ex}
Nowhere dense sets are topologically negligible.  Consequently, $\NN$ is topologically negligible in $\overline{\PCNN{\NN}}$.  
\begin{thrm}[The PCNN is Asymptotically More Expressive]\label{thrm_improvement}
Let $\sigma\in C(\rr)$ be non-polynomial, and $X=[0,1]^d$.  
Then, $\overline{\NN}$ is nowhere dense in $\overline{\PCNN{\NN}}$.
\end{thrm}  

The next result shows that $\overline{\PCNN{\fff}}$ is large in the sense that it is not \textit{separable}, this means that any dense subset thereof cannot be countable.  Examples of separable Banach spaces arising in classical universal approximation results include $C(\rrd,\rrD)$ studied in \cite{hornik1990universal,kidger2020universal}, the space of Lebesgue $p$-integrable functions $L^p(\rrd)$ on $\rrd$ studied in \cite{lu2017expressive}, or the Sobolev spaces on bounded domains in $\rrd$ as in \cite{siegel2020approximation}.  The most well-known example of a non-separable space is $L^{\infty}(\rrd)$, in which \cite{capel2020approximation} have shown that $\NN$ is not dense whenever $\sigma$ is continuous.  

\begin{thrm}[\aname s are Highly Expressive, Asymptotically]\label{thrm_improvement_largeness}
Let $\sigma\in C(\rr)$ be non-polynomial, and $X=[0,1]^d$.  Then the following hold:
\begin{enumerate}[(i)]
    \item $\overline{\PCNN{\NN}}$ is not separable,
    \item $C(X,\rrD)\subset \overline{\PCNN{\NN}}$.
\end{enumerate}
\end{thrm}  
\begin{rremark}\label{remark_comparison_UAP}
Theorems~\ref{thrm_improvement} and~\ref{thrm_improvement_largeness} formulate universal approximation theorems that express the maximum approximation capabilities of the PCNN relative to $\NN$.  
\end{rremark}
We complete this portion of our discussion here by observing that even simple piecewise continuous functions (with more than one piece) cannot be uniformly approximated by FFNNs with a continuous activation function.  
\begin{prop}[{Deep Feedforward Networks are Not Universal in $\mathcal{B}(X,\rrD)$}]\label{prop_non_universal_FFNN_cnt}
	If $\sigma\in C(\rr)$ and $K_1=[0,\frac1{2}],K_2\triangleq [\frac1{2},1]$ are given. Then,  For every $\epsilon>0$ there exists an $f_{\epsilon}$ of the form~\eqref{eq_sub_patterns} satisfying:
	$$
		\inf_{f \in \NN[1,1]} \sup_{x \in [0,1]} \, 
		\left\|
		f(x)
		-
		f_{\epsilon}(x)
		\right\|\geq \epsilon.
	$$
\end{prop}

Though the space $\mathcal{B}(X,\rrD)$ does provide a concrete environment for comparing the maximum approximation capabilities of different deep neural models, it
is nevertheless ill-suited to the approximation of piecewise continuous functions.  These are the two following reasons.  This is because approximation of any $f \in \mathcal{B}(X,\rrD)$ can not be decoupled, in the sense that any piecewise function of the form~\eqref{eq_sub_patterns} can not be approximated by approximating the $f_1,\dots,f_N$ and the $K_1,\dots,K_N$ in separate steps.  Rather uniform approximation necessitates that both must be approximated in the same step, but we cannot do this due to the discontinuous structure of our model.  Furthermore, the uniform norm does not genuinely reflect any of the "sub-pattern pattern" structure of~\eqref{eq_sub_patterns}, and rather it only views $f$ as a typical bounded function.  Accordingly, we now introduce a space of piecewise continuous functions and a mode of convergence for piecewise functions capable of detecting when piecewise functions have differing numbers of "sub-patterns" (defined below) and whose mode of convergence is amenable to a two-step optimization procedure in which trains the $f_n$ are the $K_n$ are separately.  
\subsection{PCNNs are Universal Approximators of Piecewise Continuous Functions}\label{ss_approx_theory_PWC}
We introduce our space of piecewise continuous functions as well the notion of convergence of piecewise continuous functions.  Then, we show that the $\aname$s are universal for this mode of convergence. $\aname$s are shown to optimize a certain concrete upper bound to the "distance function" of piecewise continuous functions; this upper bound reflects our proposed algorithm (described in the next section).  Lastly, we show that just like in the space $\mathcal{B}(X,\rrD)$, FFNNs with continuous activation functions are not universal in the space of piecewise continuous functions.  

\subsubsection{The Space of Piecewise Continuous Functions}
\label{s_Main_ss_Space_PCs}
The motivation for our mathematical framework for describing and approximating piecewise continuous functions can be motivated by the fact that many FFNNs implement the same continuous function \citep{kratsios2019UATs,Florian2021efficient}.  For instance, \cite{gribonval2021approximation,GUHRING2021107,Florian2021efficient} distinguishes between \textit{FFNN's (abstract) representation}, which is an $N+1$-tuple $\{W_n\}_{n=1}^{N+1}$ of composable affine functions $W_n$; where these function's parameters encode the information required to implement an FFNN once the activation function is specified.  Indeed, upon fixing $\sigma$, the \textit{realization} of an (abstract) FFNN $\{W_n\}_{n=1}^{N+1}$ is the $\hat{f}\in C(\rrd,\rrD)$ defined by:
\begin{equation}
    \hat{f}(x) = W_{N+1}\circ \sigma \bullet \dots \circ \sigma \bullet W_1(x)
    \label{eq_representation}
    .
\end{equation}
The important point here is that, once $\sigma$ is fixed, every (abstract) FFNN representation induces a unique FFNN \textit{realization} via~\eqref{eq_representation}.  However, the converse is generally false.  For instance, in \cite{Florian2021efficient}, it is shown that if $\sigma$ is the ReLU activation function, $x\mapsto\max\{0,x\}$, then any FFNN realization has infinitely many FFNN representations since ReLU layers can implement the identity map on $\rrD$.  
Analogously to FFNNs, every piecewise continuous function of the form~\eqref{eq_sub_patterns} has infinitely many different representations in terms of \textit{parts} (defined below).  Therefore, we introduce our theory of piecewise continuous functions in analogy with the above discussion, and we emphasize the distinction between a piecewise continuous function's representation and its realization.   

\paragraph{The Space of Piecewise Continuous Functions}
In a direct analogy, we define an (abstract) representation of a piecewise continuous function with \textit{$N$ parts} to be $\left((f_n,K_n)\right)_{n=1}^{N}$ where $f_n \in C(\xxx,\rrD)$ and $K_n\in \text{Comp}(X)$; for $n=1,\dots,{N}$.  The set of all abstract representations of piecewise continuous functions is therefore $\bigcup_{N\in \nn_{+}}\, [C(X,\rrD)\times \text{Comp}(X)]^N$.  
We define the \textit{realization} of any piecewise continuous function's representation $\left((f_n,K_n)\right)_{n=1}^N$ to be the following element $\rrr\left((f_n,K_n)_{n=1}^N\right)$ of $\mathcal{B}(X,\rrD)$ defined by:
\begin{equation}
    \rrr\left((f_n,K_n)_{n=1}^N\right)  (x)
        \triangleq
    \sum_{n=1}^N f_n(x)I_{K_n}(x)
        \label{eq_bounded_induced_realization}
    .
\end{equation}
The relationship~\eqref{eq_bounded_induced_realization} between a piecewise function and its representation(s) is typically non-trivial.  As the next example shows, if $N>1$ then there may be infinitely many different representations of the same piecewise function.
\begin{ex}[Non-Uniqueness of Representation]\label{ex_same_representation}
Let $f(x)=x^2I_{[0,1]}+e^xI_{[1,2]}$ and $N=4$.  For each $0<r<1$, 
$\left((x^2,[0,\frac1{r}]),(x^2,[\frac1{r},1]),(-\frac1{r^2},\{\frac1{r}\}),(e^x,[1,2])\right)$, $\left((x^2,[0,1]),(e^x,[1,2])\right)$, and $\left((e^x,[1,2]),(x^2,[0,1])\right)$ all  represent $f$. 
\end{ex}
We would like to quantify the distance between any two piecewise continuous functions which is invariant to the choice of representation.  We would equally like our "distance function" to be capable of detecting if and when two piecewise continuous functions require a different minimal number of parts to be represented.  In this way, our metric should be capable of separating any two functions $f$ and $g$ if $f$ is fundamentally more discontinuous than $g$ is.  This is made rigorous through the following; note, the infimum of $\emptyset$ is defined to be $\infty$.
\begin{defn}[Minimal Representation Number]\label{defn_minimality_number}
For any $f\in \mathcal{B}(X,\rrD)$, its \textit{minimal representation number} is:
$$
N(f)\triangleq \inf\{N:\, \exists ((f_n,K_n))_{n=1}^N\in 
\cup_{N\in \nn_+}\, [C(X,\rrD)\times \text{Comp}(X)]^N:\, f=\rrr\left(((f_n,K_n))_{n=1}^N\right)
\}.
$$
\end{defn}
Many functions $f\in \mathcal{B}(X,\rrD)$ can be arbitrarily complicated; these are beyond the scope of our analysis and can not be uniformly approximated.  Here, infinite complexity is quantified by $N(f)=\infty$.
\begin{ex}\label{ex_infinitely_complicted}
Let $f=\sum_{n\in \nn_+} nI_{[(n+1)^{-1},n^{-1}]}$; then $N(f)=\infty$.  
\end{ex}
We consider functions as in Example~\ref{ex_infinitely_complicted}, pathological and, this is because such functions do not admit any representation of the form~\eqref{eq_sub_patterns}.  The following result guarantees that if a piecewise function $f$ has a representation of the form~\eqref{eq_sub_patterns}, then it must admit at least one such minimal representation. 
\begin{prop}\label{prop_existence}
If $f\in \mathcal{B}(X,\rrD)$ admits a representation of the form~\eqref{eq_sub_patterns} then $N(f)\in \nn_+$ exists.  
\end{prop}
\begin{ex}\label{ex_gap_cnt_disc2}
If $f \in C([0,1],\rr)$ and $g=I_{[0,1]} + I_{[\frac1{3},\frac1{2}]}$ then $N(f)=1$ and $N(g)=2$.  
\end{ex}
Example~\ref{ex_gap_cnt_disc2} implies that if $N(f)>1$ then $f$ is discontinuous; however, $N(f)=1$ does not imply continuity.
\begin{ex}\label{ex_non_gap}
Let $f=I_{[0,\frac1{2}]}\in \mathcal{B}([0,1],\rr)$, then we note that $N(f)=1$.  
\end{ex}
Next, we construct a ``distance function", on the set of piecewise functions on $X$ with finite minimal representation number.  Our construction is reminiscent quotient of metric spaces (see \citep{Burago2Ivanov2001CourseMetricGeometry}) where different functions are identified as being part of the same ``equivalence class''.  We first equip $\bigcup_{N \in \nn_{+}} [C(X,\rrD)\times \text{Comp}(X)]^N$ with the following $[0,\infty]$-valued function:
\begin{equation}
    d_{\text{Step 1}}\left(
        (f_n,K_n)_{n=1}^{N}
            ,
        (f'_n,K_n')_{n=1}^{N'}
    \right)\triangleq 
    \begin{cases}
    \max_{1\leq n\leq N}\max\left\{
        \|f_n-f_n'\|_{\infty}
            ,
        d_H(K_n,K_n')
    \right\}
         & : N= N'\\
    \infty &: N\neq N'
    \end{cases}
    \label{defn_metric_step1}
    .
\end{equation}
Next, we use the map $d_{\text{Step1}}$ to construct a "distance function" on $\mathcal{B}(X,\rrD)$ which expresses the similarity of any two function $f,g\in \mathcal{B}(X,\rrD)$ therein in terms of their most similar and efficient yet compatibility representations.  
A representation $F:={(f_n,K_n)}_{n=1}^{N}\in \cup_{N\in \nn_+} [C(X,\rrD)\times \text{Comp}]^{N}$ is understood as being {\color{PineGreen}{\textit{efficient}}} if $\#F\triangleq N\approx N(f)$.  

Likewise, two representations $F$ and $G$, respectively representing $f$ and $g$, are thought of as being \textit{compatible} if $\#F=\#G$ (thus, $d_{\text{Step 1}}(F,G)<\infty$) and 
{{\color{MidnightBlue}{\textit{similar}}}} if $d_{\text{Step 1}}(F,G)\approx 0$.  
In this way, we may quantify the "distance" between any two bounded functions in $\mathcal{B}(X,\rrD)$ by searching through all compatible representations, which simultaneously penalizes any such representation for its {\color{PineGreen}{inefficiency}}.  We define this "distance function", or divergence, as follows.
\begin{defn}[Piecewise Divergence]\label{defn_PC_Div}
The \textit{piecewise divergence} between any two $f,g \in \mathcal{B}(X,\rrD)$ is defined as:
\begin{equation}
    D_{PC}
    (f|g)
    \triangleq 
    \underset{
    R(F) = f,\,
    R(G)=g
    }{
    \inf
    }
    \,
    \underbrace{
    d_{\text{Step 1}}\left(
        F
        ,
        G
    \right)
    }_{
    {\color{MidnightBlue}{\text{Representation Similarity}}}
    }
    +
    \underbrace{
    |\#F - N(f)|
    +
    |\#G - N(f)|
    .
    }_{
    {\color{PineGreen}{
    \text{Cost of Comparing Inefficient Representations}
    }}
    }
    \label{defn_metric_step2}
\end{equation}
\end{defn}
We call $D_{PC}$ a divergence since even if is symmetric, unlike a metric, it need not satisfy the triangle inequality.  Divergences form a common method of quantifying distance in machine learning. Notable examples given by Bregman divergences (see \cite{SiDachenGeng2010BregmanRegularization} for applications in regularization, \cite{ClusteringBregman2005} for applications in clustering, and \cite{InfiniteDimBregman} for applications in Bayesian estimation).  Examples of a divergences are the Kullback-Leibler divergence of \cite{KullbackLeiblerOriginal1951} and, more generally, $f$-divergences which quantify the information shared between different probabilistic quantities \citep{amari2000methods,JMLR:v22:20-867}.  
Similarly, the divergence $D_{PC}$ allows us to define the space of piecewise continuous functions by identify all ``finitely complicated bounded functions'' with their ``subpatterns'' and the ``parts'' on which they are defined.  

Let $\mathcal{B}_+(X,\rrD)$ consist of all $f\in \mathcal{B}(X,\rrD)$ with $N(f)<\infty$.  We formalize our ``space of piecewise continuous functions'', denoted by
$PC(X,\rr^D)$, to be the set of equivalence classes of $f\in \mathcal{B}_+(X,\rrD)$ where any two functions $f,g \in \mathcal{B}_+(X,\rrD)$ are identified if $D_{PC}(f,g)=0$ and $D_{PC}(g,f)=0$.

Unlike $\mathcal{B}(X,\rrD)$, the space $PC(X,\rrD)$ is well-suited to our approximation problem.  This is because, for any $f \in PC(X,\rrD)$, the function $D_{PC}(f|\cdot)$ is upper-bounded by the error of approximating each $f_n$ and each $K_n$ individually (using some optimally efficient $((f_n,K_n))_{n=1}^{N(f)}$ approximate representation of $f$).  Most notably, unlike the uniform metric on $\mathcal{B}(X,\rrD)$, the approximation error in approximating each $f_n$ and each $K_n$ is fully decoupled from one another in this upper-bound of $D_{PC}(f|\cdot)$.  Thus, we may approximate each $f_n$ and then each $K_n$ in two separate steps.
\begin{prop}[{Decoupled Upper-Bound on $D_{PC}(f|\cdot)$}]\label{prop_decoupling_lemma}
Fix $f \in PC(X,\rrD)$ and suppose that $((f_n,K_n))_{n=1}^{N(f)}\in [C(X,\rrD)\times \text{Comp}(X)]^{N(f)}$ represents $f$.  For any $((\hat{f}_n,\hat{K}_n))_{n=1}^{N(f)} \in [C(X,\rrD)\times \text{Comp}(X)]^{N(f)}$ the following holds: 
\begin{equation}
D_{PC}\left(f\,\middle|\,
\sum_{n=1}^{N(f)}
\hat{f}_n I_{\hat{K}_n}
\right)
\leq 
\sum_{n=1}^{N(f)}\, \|f_n-\hat{f}_n\|_{\infty} + \sum_{n=1}^{N(f)} \, d_H(K_n,\hat{K}_n)
\label{eq_prop_decoupling_lemma}
.
\end{equation}
\end{prop}
The decoupled upper-bound for the piecewise divergence $D_{PC}$ derived in Proposition~\ref{prop_decoupling_lemma} will now be used to show that $\anames$ are universal approximators of piecewise continuous functions.   

As a final point of interest and motivation for our main result, we demonstrate that sequences of FFNNs in $\NN[d,D]$ generally do not converge, with respect to the piecewise divergence, to piecewise functions with at-least two parts.  This is because the PC divergence between any two functions $f$ and $g$ is lower-bounded by $|N(f)-N(g)|$.  In particular, since every feed-forward network is in $C(X,\rrD)$, then the result follows from Example~\ref{ex_gap_cnt_disc2}.  
\begin{prop}[FFNNs are Not Universal Piecewise Continuous Functions]\label{prop_Disc_Gap}
Fix $\sigma \in C(\rr)$.  For each $f\in PC(X,\rrD)$, if $N(f)>1$ then:
$
1\leq \inf_{\hat{f}\in \NN[d,D]}\,D_{PC}(f|\hat{f})
$.  
Furthermore, if $X=[0,1]$ then, $f=I_{[0,1]}+I_{[\frac1{3},\frac1{2}]}\in PC(X,\rr)$ and $N(f)>1$.  
\end{prop}

\subsubsection{Approximating Piecewise Continuous Functions by {$\anames$}}\label{ss_UAT}
Our results, we require the following regularity condition on $\sigma\in C(\rr)$ introduced in \cite{kidger2020universal}.  
\begin{ass}\label{ass_KL}
$\sigma\in C(\rr)$ is not affine.  There is an $x_0\in \rr$ at which $\sigma$ is continuously-differentiable and $\sigma'(x_0)\neq 0$.  
\end{ass}
We may now state our main result, which shows that $\anames$ is a universal approximator of piecewise continuous function with respect to the piecewise continuous divergence, $D_{PC}$.  In particular, $\anames$ is strictly more expressive that FFNNs since Proposition~\ref{prop_Disc_Gap} proved that FFNNs are not universal with respect to $D_{PC}$.
\begin{thrm}[{$\aname$s are Universal Approximators of Piecewise Continuous Functions}]\label{thrm_UAT_PC}
Fix $\sigma$ satisfying Assumption~\ref{ass_KL}.  
Let $f\in PC([0,1]^d,\rrD)$, $0<\gamma<1$, and $0<\epsilon$.  
There exist a $\hat{f}\in PC(X,\rrD)$ with representation $\hat{F}\triangleq \left(
(\hat{f}_n,\hat{K}_n)
\right)_{n=1}^N$ where each$\hat{f}_n\in \NN[d,D]$ of width at-most $d+D+2$, $\hat{K}_n\triangleq \left\{x\in X:\, I_{(\gamma,1]}\circ \sigma_{\operatorname{sigmoid}}\circ \hat{c}(x)_n\right\}$ and where $\hat{c}\in \NN[d,N(f)]$ has width at-most $d+N(f)+2$, such that:
\begin{equation}
    D_{PC}(f|\hat{F})<\epsilon
    .
    \label{eq_thrm_UAT_PC_bound}
\end{equation}
\end{thrm}
A key step towards establishing Theorem~\ref{thrm_UAT_PC}, which we now highlight, is deriving the following fact that our \textit{deep zero-sets} are universal in $\text{Comp}(X)$ with respect to the Hausdorff metric $d_H$, thereon.  This result is the first result describing the universal approximation of compact sets defined by a deep neural model.  
\begin{thrm}[Deep Zero-Sets are Universal Compact Sets]\label{thm_generic_sets}
Let $\sigma$ satisfy Assumption~\ref{ass_KL}, $\emptyset\neq K_1,\dots,K_N\subseteq X$ be compact, $0<\epsilon$, and set $\gamma\triangleq \sigma_{\text{sigmoid}}^{-1}(2^{-1}\epsilon)$.  There is an $\hat{c}\in {NN}_{d,N}^{\sigma}$ of width at-most $d+N+2$ for which the deep zero-sets
$
\hat{K}_n\triangleq \left\{x\in X:\, I_{(\gamma,1]}\circ \sigma_{\operatorname{sigmoid}}\circ \hat{c}(x)_n\right\}
$
simultaneously satisfy:
\begin{equation}
    \max_{n=1,\dots,N}\, d_{H}(K_n,\hat{K}_n)\leq \epsilon
    \label{eq_thm_generic_sets}
    .
\end{equation}
\end{thrm}

We introduce our parallelizable procedure for training the $\anames$ and we investigate its theoretical properties.  
\subsection{The Training Meta-Algorithm}\label{s_Main_ss_MetaAlgo}
Let $\xx$ be a non-empty set of training data in $\rrd$ and let $L:\rrd\rightarrow [0,\infty)$ quantify the learning problem
\begin{equation}
    \inf_{\hat{f}\in \text{PCNN}} \sum_{x \in \xx} L\left(
\hat{f}(x)
\right)
.
\label{eq_learning_problem}
\end{equation}
Even if $L$ is smooth, we cannot use gradient-descent type methods to optimize~\eqref{eq_learning_problem} (in-sample on $\xx$) since the map $I_{(\gamma,1]}$ makes each $\hat{K}_n$ into a discontinuous function, and by extension $\hat{f}$ need not be differentiable (even in the generalized sense of \cite{Clarke1975generalizedGradientsAndApplications}).  Meta-Algorithm~\ref{metaalgo_GET_PCNN} proposes an approach for training the $\aname$ which avoids passing gradient updates through the indicator function $I_{(\gamma,1]}$ in $\hat{K}_n$ by decoupling the training of the $\{\hat{f}_n\}_{n=1}^N$ and the $\{\hat{K}\}_{n=1}^N$.  

\makeatletter
\renewcommand{\ALG@name}{Meta-Algorithm}
\renewcommand{\listalgorithmname}{List of \ALG@name s}
\makeatother
\begin{algorithm}
\caption{{Train $\aname$.}}
\label{metaalgo_GET_PCNN}
\begin{algorithmic}[1]
    
%
		\STATE {\bfseries Input:}  \textit{Training data} $\xx$, Number of Parts $N$, Depth parameters $J_1,J_2,W\in \nn_+$, FFNN training subroutine $\mbox{GET\_FFNN}$, $\mbox{GET\_PARTITION}$ subroutine.
		\label{alg1_step0}
		\STATE {\bfseries Output:} \textit{Trained $\aname$ } $f^{\star}\triangleq \sum_{n\leq N} f_n^{\star}(\cdot)\cdot I_{K_n^{\star}}(\cdot)$
        \STATE $\{\xx_n\}_{n=1}^N  \leftarrow \operatorname{GET\_PARTITION}$ $(\xx )$  \label{alg1_step1}
%
         \FOR{$n\leq N$}
            \STATE $f_n^{\star} \leftarrow \mbox{GET\_FFNN}(\mathbb{X}_n,L,J,W)$
        \ENDFOR
		\FOR{$x\in \xx$
		}
    		\STATE $l_{n,x}\triangleq I\left(
    		L(f^{\star}_n(x))\leq \min_{m\leq N}
    		L(f^{\star}_m(x))
    		\right)$
		\ENDFOR
           \STATE 
           $c^{\star} \leftarrow 
           \operatorname{GET\_FFNN}(\mathbb{X},H(\cdot|l_{n,\cdot}),J_2,W)
           $
			\FOR{$n\leq N$}
			\STATE Define deep zero-sets $K_n\triangleq \{x\in \rr^n:\, c^{\star}_n(x)\leq 2^{-1}\}$
			\ENDFOR
\end{algorithmic}
\end{algorithm}
Step $1$ (line $3$)  initializes partitions of the training data, according to an extraneously given subroutine $\operatorname{GET\_PARTITION}$, an example of which we propose in Subroutine~\ref{subroutine_GET_PARTITION} below.  
Step $2$ (lines $4$-$6$) optimizes the networks $\hat{f}_n$.   

Step $4$ (lines $7$-$9$)  identifies which optimized sub-pattern $\hat{f}_n$ best performs on any given input and adjusts the partitions.   Step $5$ (lines $10$-$12$) interprets the $\{l_{n,x}\}_{n,x}$ in $\{0,1\}$ as labels and trains the $\hat{c}_n$ defining the $\hat{K}_n$ (see Definition~\ref{defn_PCNNs}) as a classifier predicting when $f^{\star}_n$ offers the best \textit{cross-entropy}:
	$
	H\left(
		y|l_{n,x}
		\right)\triangleq 
	l_{n,x}\ln\left(
			y
			\right)
			(1-l_{n,x})\ln\left(
			1-y
			\right),
	$
	amongst the $\{f^{\star}_n\}_{{n=1}^N}$ (justified by Theorem~\ref{thrm_partition_learner} below).

Though there are many possible clustering algorithms which can stand-in for the subroutine $\operatorname{GET\_PARTITION}$ in Meta-Algorithm~\ref{metaalgo_GET_PCNN}, we present a novel geometric option with desirable properties.  
\subsubsection{Initializing and Training the Deep Zero-Sets}
\label{sss_how_to_initialize}
Suppose that we can operate under the "geometric priors" that nearby points tend to belong to the same part $K_n$ and, given $N$, the parts $\{K_n\}_{n=1}^N$ are as efficient as possible.  This means that $\{K_n\}_{n=1}^N$ should partition $X$ while having the smallest possible boundaries. In particular, if $\xx$ is representative of $\{K_n\}_{n=1}^N$, then $\operatorname{GET\_PARTITION}$ seeks to partition the $\xx$ into $N$ parts $\{\xx_n\}_{n=1}^N$ while minimizing the distance between every pair of data-points in each $\xx_n$.  

This problem is known as the min-cut problem and it is a well-studied problem in computer science.  In particular, the authors of \cite{dahlhaus1994complexity} show that this is an NP-hard problem.  Nevertheless, exploiting the max-flow min-cut duality (see \cite{MaxflowmMincutTheoremImprovedPulat1989}) a randomized-polynomial time algorithm which approximately solves the min-cut problem with high probability is developed in \cite{Bartalmetricapprox}.  
Furthermore, the algorithms tends to assign nearby points to the same part, where the probability of this happening depends linearly on the distance between those two points.  

One stand-in for $\operatorname{GET\_PARTITION}$, which we describe in Subroutine~\ref{subroutine_GET_PARTITION}, modifies the procedure of \cite{Bartalmetricapprox} to fit our setting.  Key properties of our variant are described in Proposition~\ref{cor_dat_driven_opt_alloc}.  In particular, Proposition~\ref{cor_dat_driven_opt_alloc} (v) states that with high-probability, two nearby data-points in $\xx$ are mapped to the same sub-pattern.  

We use $\Delta(\xx) \triangleq \min_{x,y \in \mathbb{X};\, x\neq y}\frac1{2}\|x-y\|$ to denote the minimum distance between any pair of training data-points  and $\bar{\Delta}(\xx)$ the mean distance between distinct training data, given by $\bar{\Delta}(\xx)\triangleq \frac1{\#\{(x,y)\in \xx^2:\, x\neq y\}}\sum_{(x,y)\in \xx^2:\, x\neq y}\|x-y\|$. In Algorithm \ref{subroutine_GET_PARTITION}, the randomness arises from $\alpha$ which contributes to the random radius in Line 7 that forms the data partitions. The minimum portion of the data required in each partition is denoted by $q$. Line 3, shuffles the training data. Line 7 use the random radius defined by $\alpha$ to form the partitions of the data and line 8 extends the partition of the data to forms parts of the input space.  Lines 10-14 ensure each part is not too large relative to the others.

\makeatletter
\renewcommand{\ALG@name}{Subroutine}
\renewcommand{\listalgorithmname}{List of \ALG@name s}
\makeatother
	\begin{algorithm}[ht]
	\caption{$\operatorname{GET\_PARTITION}$.}
	\label{subroutine_GET_PARTITION}
\begin{algorithmic}[1]
	\STATE {\bfseries Input:} Training data $\xx$, $q \in (0,1)$ 
	\STATE Sample $\alpha$ uniformly from $(\frac1{2},\frac1{4}]$\;
	\STATE Pick a bijection $\pi:\{1,\dots,\# \xx\}\rightarrow \xx$\; 
	\STATE Compute $\Delta(\xx)$ and $\bar{\Delta}(\xx)$
	\STATE  $\xx^{'} \leftarrow \xx$
	 \LOOP
		\STATE $ \xx_n\triangleq \left\{z \in \xx^{'}:\, \|z-\xx^{'}_0\|<\alpha \bar{\Delta}(\xx)\right\}$
		\STATE $ K_n^{\xx}\triangleq \left\{z \in \rrd:\,(\exists x \in \xx_n)\, \|x-z\|\leq
		\Delta(\xx)\right\}$
		\STATE $ \xx^{'} \leftarrow \xx^{'} - \xx_n$
	    \IF {$\ \frac{\# \xx^{'}}{\# \xx} \leq q$}
	    \STATE $ \xx_n\triangleq \xx^{'}$
		\STATE $ K_n^{\xx}\triangleq \left\{z \in \rrd:\,(\exists x \in \xx_n)\, \|x-z\|\leq \Delta(\xx)\right\}$
		\STATE BREAK
	    \ENDIF
    \ENDLOOP
 	\STATE {\bfseries Output:} 
 	$\{\xx_n\}\triangleq \left\{\xx_n\neq \emptyset \right\}$,
	$N\triangleq \#\{\xx_n\}$,
	$\{K_n^{\xx}\}\triangleq \left\{K_{n}^{\xx}\neq \emptyset \right\}$
\end{algorithmic}
\end{algorithm}
\begin{prop}[{Properties of Subroutine~\ref{subroutine_GET_PARTITION}}]\label{cor_dat_driven_opt_alloc}
		Let $\xx$, $q=1$, and $\left\{K_{n}^{\xx}\right\}_{n=1}^N$ be as in Subroutine~\ref{subroutine_GET_PARTITION}, and fix $\hat{f}_1,\dots,\hat{f}_n \in \NN$.  Set $f=\sum_{n=1}^N f_n I_{K_n^{\xx}}$.   
		The following hold:
		\begin{enumerate}
			\item[(i)] For $x_1,x_2 \in \xx$, 
			$
			\pp\left(
				\min_{n=1,\dots,N}
				\max_{i=1,2}\|\hat{f}(x_i)-\hat{f}_n(x_i)\|
				    =
				0
				\right)
				\geq 1- \frac{8(\ln(\# \xx)+1)}{\bar{\Delta}(\xx)} \|x_1-x_2\|
				,
			$
			\item[(ii)] Subroutine~\ref{subroutine_GET_PARTITION} terminates in polynomial time,
			\item[(iii)] For $n\leq N$, $\intt[K_n^{\xx}]\neq \emptyset$,
			\item[(iv)] If $n\neq m$ then $K_n^{\xx}\cap K_m^{\xx}=\emptyset$ 

		\end{enumerate}
\end{prop}
Our final result guarantees that, given trained models $\{f_n^{\star}\}$, there exists $\{K^{\star}_n\}\subset \text{Comp}(X)$ which optimizes the \aname's performance with arbitrarily high-probability.  We quantify this by a fixed Borel \textit{probability measure} $\pp$ on $X$.  
	\begin{thrm}[Existence: Performance Optimizing Partition]\label{thrm_partition_learner}
Fix $\{\hat{f}_n\}_{n=1}^N \in \NN$ and $L\in C(\mathbb{R}^D,[0,\infty))$ for which $L(0)=0$.  There exists a compact subset $X_{\delta,\pp}\subseteq X$ and a partition of $X_{\delta,\pp}$ satisfying:

		\begin{enumerate}[(i)]
			\item $\pp(\kkk_{\delta,\pp})\geq 1-\delta$,
			\item For $n\leq N$ and every $x \in K_n^{\star}$,
			$
			L(\hat{f}_n(x)) = \min_{m\leq N} L(\hat{f}_m(x))
			.
			$
		\end{enumerate}
		Moreover, if $\xx\subseteq \kkk_{\delta,\pp}$ and $L(f^{\star}_n(x))< \min_{n\neq \tilde{n},\,\tilde{n}\leq N}L(f^{\star}_{\tilde{n}}(x))$ for each $n\leq N$, $x \in \xx_n$ then, $\xx_n\subseteq K_n^{\star}$.  
	\end{thrm}
\begin{rremark}[Approximation of Optimal Partition]\label{cor_approximability}
Applying Theorem~\ref{thrm_UAT_PC}, we conclude that for every $\epsilon>0$ there is a $\aname$ $\hat{f}$ such that 
$
    D_{PC}(\sum_{n=1}^N\hat{f}_nI_{K_n^{\star}}|\hat{f})<\epsilon
    ;
    $ where the $\hat{f}_1,\dots,\hat{f}_N$ are from Meta-Algorithm~\ref{metaalgo_GET_PCNN} (4-6).
\end{rremark}	

\section{Numerical Experiments}\label{s_Implementation}

We evaluate the $\aname$'s performance on three different regression tasks.  The goal of our experiments is twofold.  The first objective is to show that the $\aname$ trained with Meta-Algorithm~\ref{metaalgo_GET_PCNN} better approximates the $\{K_n\}_{n=1}^N$ defining the piecewise function compared to the the benchmark models. The second goal is to show that the model offers a predictive advantage when the function being approximated is discontinuous.  

\subsection{Implementation Details}\label{s_Implementation_ss_Details}
In the following, We implement $\aname$, trained with Algorithm~\eqref{metaalgo_GET_PCNN} and subroutine~\ref{subroutine_GET_PARTITION}, against various benchmarks.  The first benchmark is the FFNNs, which we use to evaluate the predictive performance improvement obtained by turning to PCNNs and utilizing a discontinuous layer.  
The second class of benchmarks focuses on the effectiveness of the $\aname$'s structure itself.  We consider two naive alternatives to the proposed model design.  Since $\aname$ can be viewed as an ensemble model, we compare its predictive performance against a bagged model (FFNN-BAG) wherein the user trains $N$ distinct feed-forward networks $\hat{f}_1,\dots\hat{f}_N$ on the deep zero-sets generated using subroutine~\ref{metaalgo_GET_PCNN} which are summed together to construct the bagged model $\hat{f}_{\text{FFNN}_{\text{BAG}}}(x)\triangleq \sum_{n=1}^N \hat{f}_n(x)$.  This benchmark has the benefit of distinguishing parts of the inputs space via the $\{K_n\}_{n=1}^N$, instead of naively grouping them into one estimator.  Next, to evaluate the effectiveness of the deep partitions, we consider the (FFNN-LGT) model, which is identical to the $\aname$ in structure except that the deep classifier $\sigma$ defining the deep partitions $\hat{K}_n$ is replaced by a simple logistic classifier.

The model quality is evaluated according to their test predictions and the learned model's complexity/parsimony.  Prediction quality is quantified by mean absolute error (MAE), mean squared error (MSE), and mean absolute percentage error (MAPE), each evaluated on the test set.   The models are trained by using the mean absolute error (MAE).  The model \textit{complexity} is assessed by the total number of parameters in each model ($\#$Par), the average number of parameters processing each input ($\#$Par/x) and the training times, either with parallelization (P. time) or without (L. time).  

Our numerical experiments emphasize the scope and compatibility of Meta-Algorithm~\ref{metaalgo_GET_PCNN} with most FFNN training procedures.  Fittingly, the experiments in Figure~\ref{fig_Motivational_Demonstration_of_PCNNs} and Sections~\ref{s_Implementation_ss_Exp_Part} take $\operatorname{GET\_FFNN}$ to be the ADAM stochastic optimization algorithm of \cite{ADAM2015KingmaB14} whose behaviour is well-studied \cite{AdamAndBeyond}. The experiments in Sections~\ref{s_Implementation_ss_RW_Disc} and~\ref{s_Implementation_ss_synthetic_known_partition} train the involved FFNNs by randomizing all their hidden layers and only training their final ``linear readout'' layer.  This latter method is also well understood  \citep{gelenbe1990stability,louart2018random,NEURIPS2019_5481b2f3,cuchiero2020deep,gonon2020approximation}).   

In the latter experiments, we also benchmark PCNNs against a deep feedforward network with randomly generated hidden weights and linear readout trained with ridge-regression (FFNN-RND).  The FFNN-RND and FFNN benchmarks put the speed/precision trade-off derived from randomization into perspective.  This allows us to gauge the $\aname$ architecture's expressiveness as it is still the most accurate method even after this near-complete randomization; in comparison, the FFNN-RND method's predictive power will reduce when compared to the FFNN trained with ADAM.  


\subsection{Learning Discontinuous Target Functions}
\label{s_Implementation_ss_RW_Disc}
It is well known that the returns of most commonly traded financial assets forms a discontinuous trajectory \citep{lee2008jumps,ContTankov2004FinancialJumpProcesses,FilipovicMartin2020PolynomialJumpDiffusions}; the origin of these ``jump discontinuities'' are typically abrupt ``regime switches'' of the underlying market dynamics caused by news or other economic and behavioural factors.  The magnitude of discontinuous behaviour depends on the idiosyncrasies of the particular financial assets.  The next experiments illustrate that $\anames$ can effectively model discontinuous function, with varying degrees of discontinuities, by illustrating that they can replicate the returns of assets with different levels of volatility; i.e.: the degree to which any asset's returns fluctuate.    

\subsubsection{Mild Discontinuities: SnP500 Market Index Replication}\label{s_Implementation_ss_RW_Disc__ss_Main_ss_MetaAlgoation_snp_replication}
In this experiment we train the $\aname$, along with the benchmarks, to predict the next-day SnP500 market index's returns using the returns of all its constituents; i.e. the stocks of the 500 largest companies publically traded in the NYSE, NASDAQ, Cboe BZX exchanges.  The dataset consists of $2$ years of daily closing data ending on September $9^{th}$ $2020$.  The test set consists of the final two weeks.  All returns are computed with the daily closing prices.  

\begin{figure}[H]
\centering
\begin{floatrow}
\ffigbox{%
  \centerline{\includegraphics[scale=0.1]{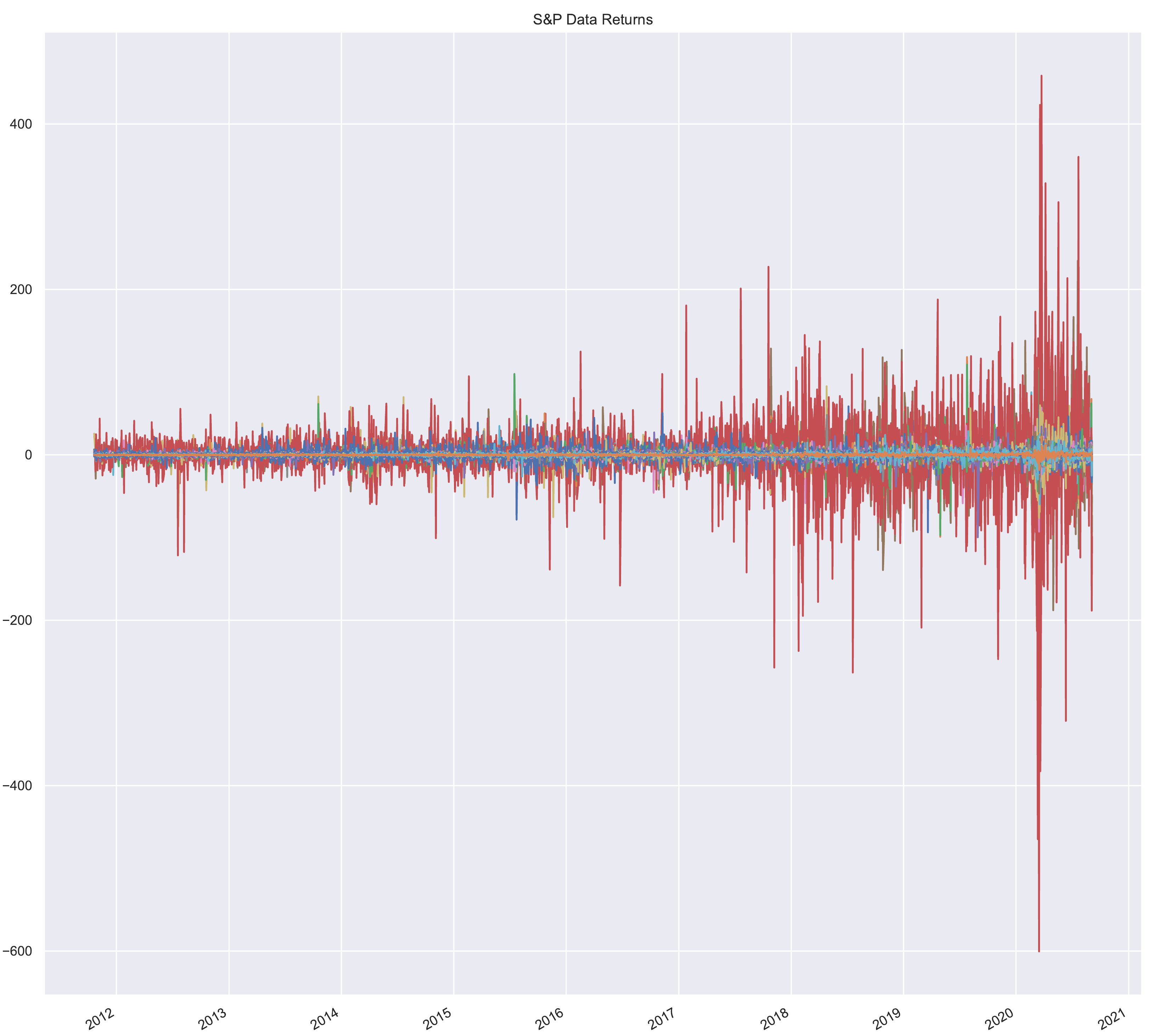}}
}{%
  \caption{Four Year Returns of SnP500 Constituents' Returns}%
  \label{fig_2_year_returns}
}
\capbtabbox{%
\begin{adjustbox}{width=\columnwidth,center}
\begin{tabular}{lrrlrr}
\toprule
{} &        MAE &  P. Time & L. Time &       $\#$Par/x &     $\#$Par \\
\midrule
FFNN     &    9.0e+2 &  - &       8.0e+2 &     6.0e+3 &     6.0e+3 \\
FFNN-RND & 1.3e+4 &   - &       8.7e-3  &     6.0e+3 &      6.0e+3 \\
FFNN-BAG &    9.0e+2 &   1e-1 &  2.4e-1 &  5.6e+6 &  5.6e+6 \\
FFNN-LGT &    9.1e+2 &   2.5e+1 &  2.7e-1 & 6.1e+6 &  5.5e+6 \\
PCNN     &    8.7e+2 &   9.1e+1 &  9.2e+1 &   5.1e+4 &  5.4 e+5 \\
\bottomrule
\end{tabular}
\end{adjustbox}
}{%
    \caption{Predictive Performance and Complexity Metrics}
    \label{SnP_table}
}
\end{floatrow}
\end{figure}

\begin{figure}[ht]

\centering
\subfigure[MAE]{\label{fig_SnP_MAE}\includegraphics[width=14em]{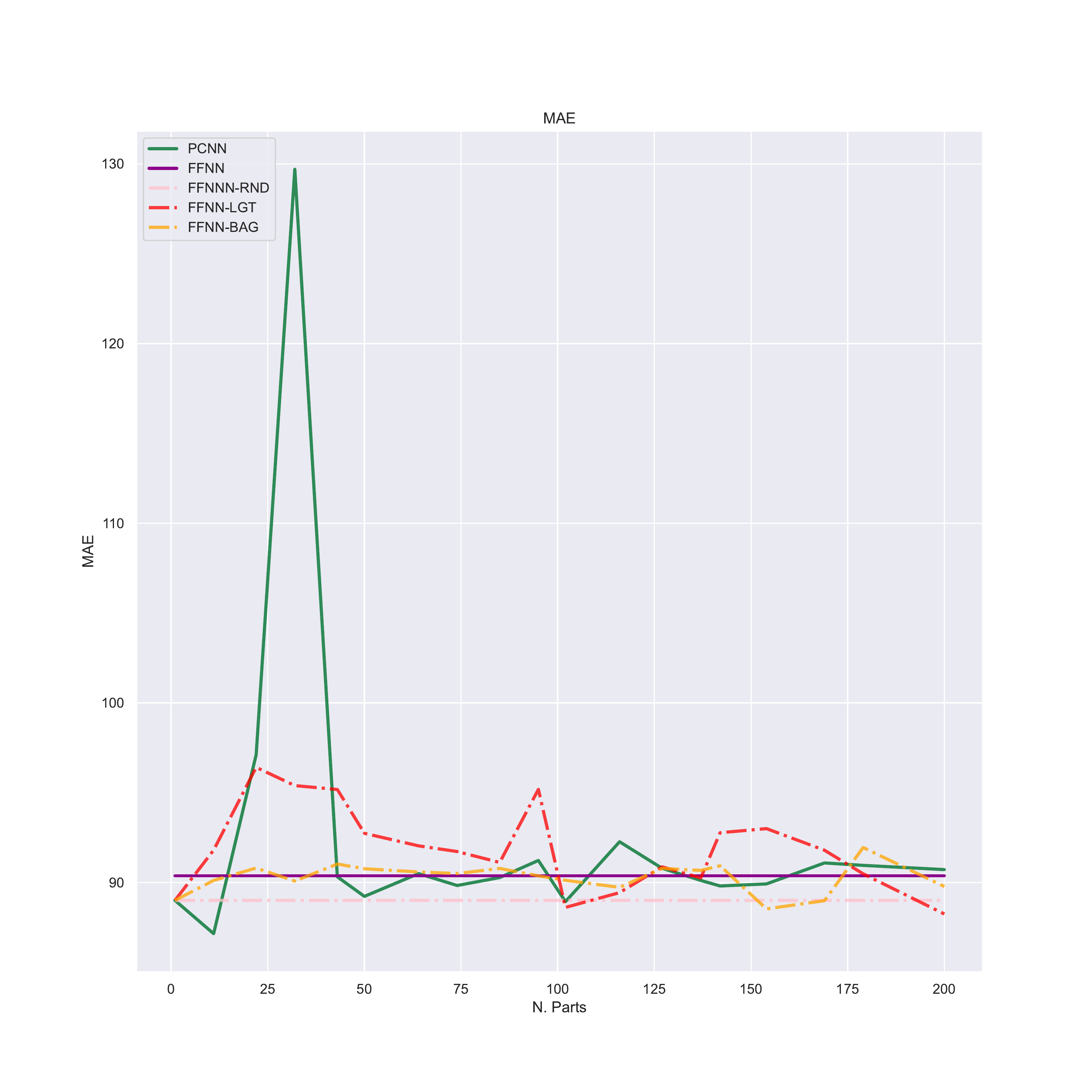}}
\subfigure[P. time]{\label{fig_SnP_PTime}\includegraphics[width=14em]{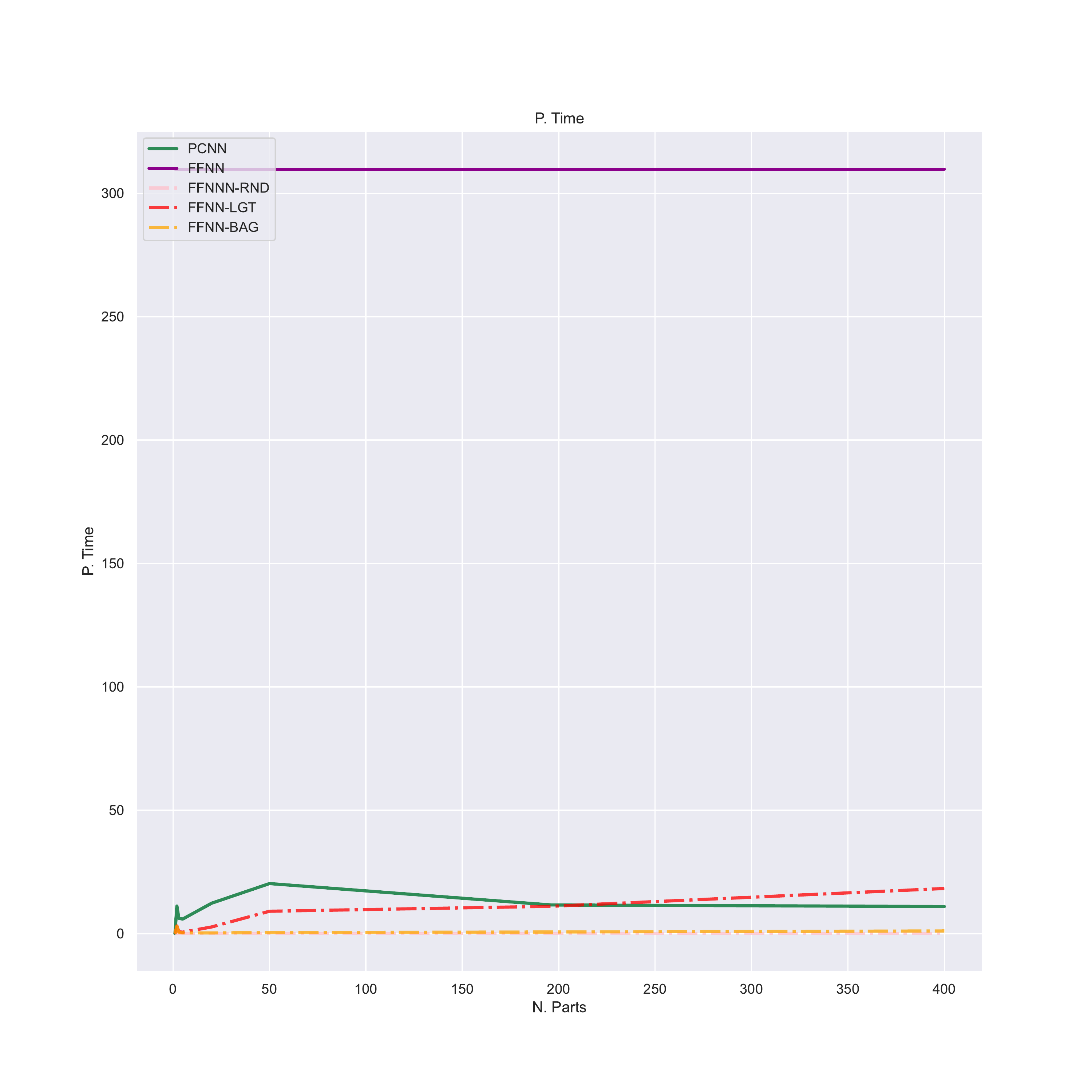}}
\subfigure[Prop. Total Neurons/Input]{\label{fig_SnP_NActiveNeurons}\includegraphics[width=14em]{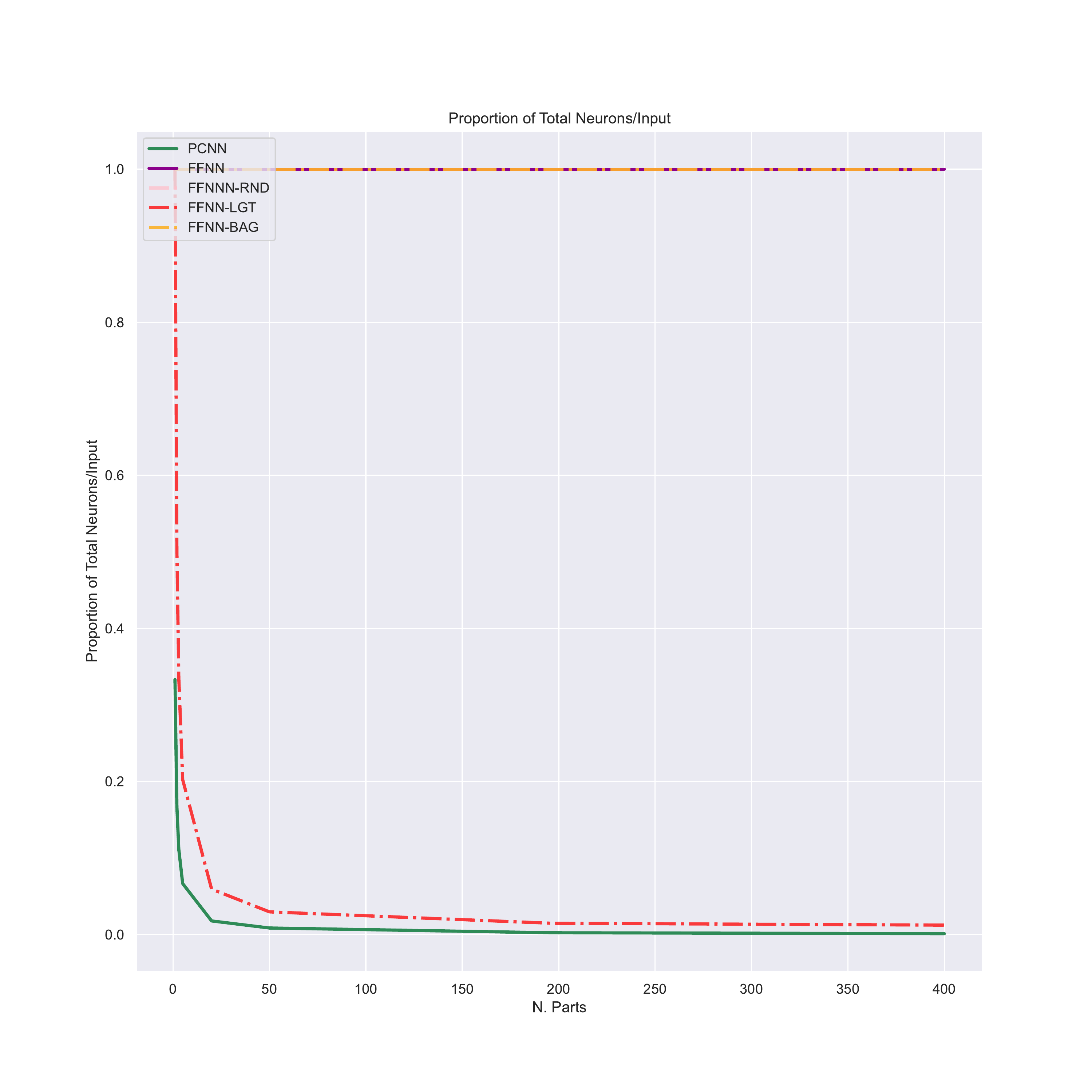}}

\caption{Performance as Function of Number of Parts with Variable Number of Neurons.}
\label{fig_Ablation___Synthetic_Untied}
\end{figure}

Since the principle parameter, extending FFNNs to PCNNs, is the integer $N$, controlling the number of parts and sub-patterns thereon.  We study the effect of varying the number of parts when training a PCNN using Algorithms~\ref{metaalgo_GET_PCNN} and~\ref{subroutine_GET_PARTITION}. Figures~\ref{fig_SnP_MAE}-\ref{fig_SnP_NActiveNeurons} explore the effect of varying $N$ on the considered performance metrics.  We find that the $\aname$ achieves the best performance amongst the considered models while relying on the smallest number of trainable parameters.  Thus, $\aname$ is the most efficient model.  Throughout this ablation experiment, the PCNNs deployed in Figures~\ref{fig_SnP_MAE}-\ref{fig_SnP_NActiveNeurons} are forced to have a comparable number of active neurons, with a fixed minimum width to ensure expressibility. This is necessary, for example, even in the case when there is a single-part whereon PCNNs essentially coincide with FFNNs \citep{johnson2018deep,park2020minimum}.


Figures~\ref{fig_SnP_MAE} and~\ref{fig_SnP_PTime} show that, once enough parts have been built into the PCNN, it outperforms the feedforward models.  From figure~\ref{fig_SnP_PTime}, we also see that the parallelizability and randomization of Subroutine~\ref{subroutine_GET_PARTITION}'s $3^{rd}$ step enables a relatively small increase in parallelized training times compared to the FFNN, even when $N$ increases.  Furthermore, since each $\hat{f}_n$ are progressively narrowed as $N$ increases, then the training time further accelerates as the $\hat{f}_n$ are built using progressively fewer neurons.  

Figure~\ref{fig_SnP_PTime} shows that, P. time increases as the number of parts defining $\aname$ do, this is because training $\hat{c}$ and Subroutine~\ref{subroutine_GET_PARTITION} is not parallelizable and scale in $N$; albeit not dramatically.  Thus, the number of parts defining the PCNNs' predictive power, as expressed through the MAE and MSE losses, rises before tapering off.  This is because a higher number of parts allows more regions of discontinuities to be captured, but since our experiment fixed the total number of neurons defining the PCNNs then each sub-pattern can become too narrow to support expressivity as the number of parts becomes large.  

\begin{figure}[H]
\centering
\begin{floatrow}
\ffigbox{%
  \centerline{\includegraphics[scale=0.1]{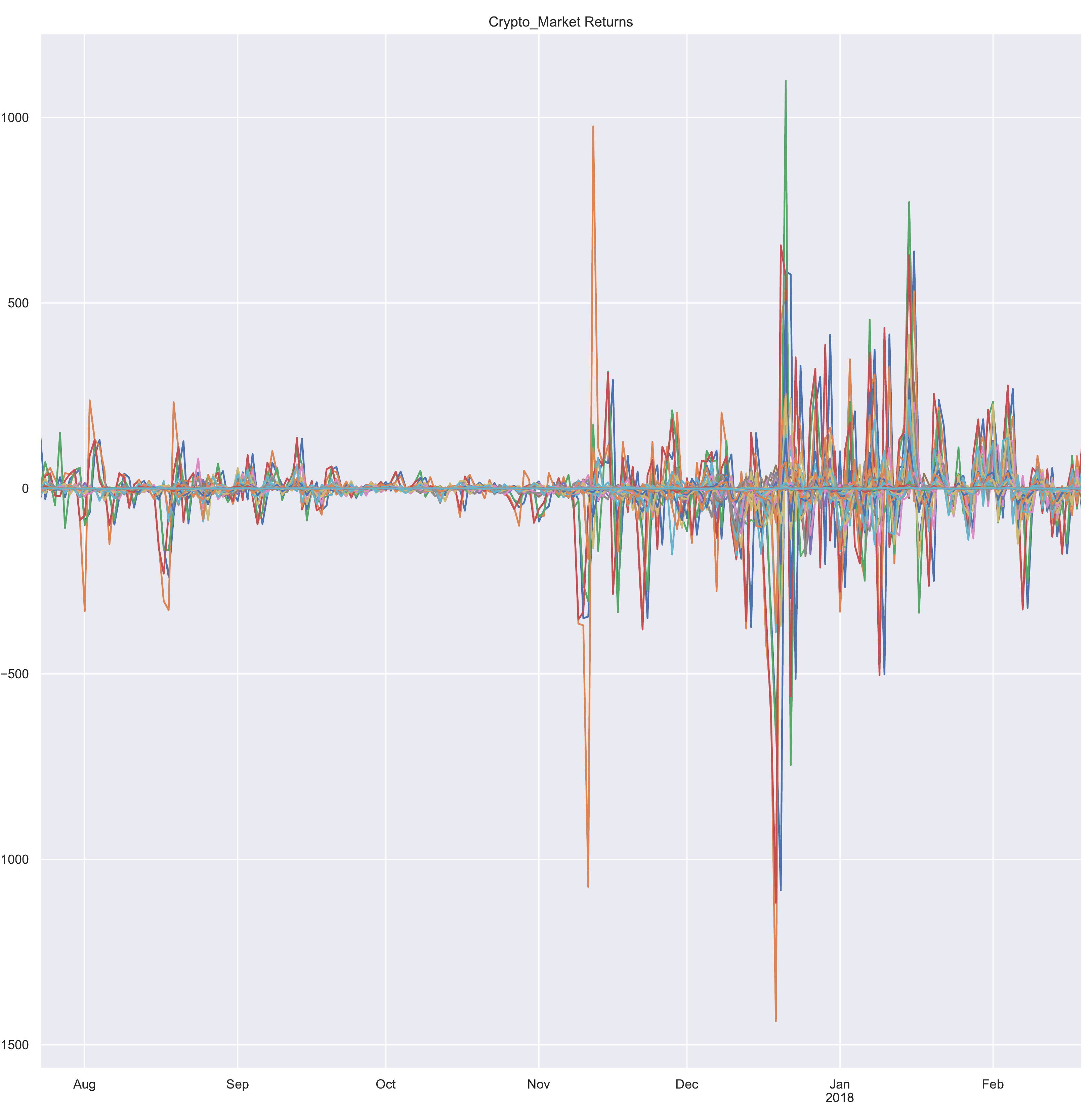}}
}{%
  \caption{Two Year Cryptocurrency Returns}%
  \label{fig_2_year_returns_crypto}
}
\capbtabbox{%
\begin{adjustbox}{width=\columnwidth,center}

\begin{tabular}{lrrlrr}
\toprule
{} &        MAE &  P. Time & L. Time &        $\#$Par/x  \\
\midrule
FFNN     &    9.6e+2 &  - &       1.2e+1 &     1.2e+4  \\
FFNN-RND & 7.6e+4 &   - &       2.7e-3 &     1.2e+4 \\
FFNN-BAG &   4.5e+3 &   3.1e-2 &  3.9e-1 &   4.87+6  \\
FFNN-LGT &   1.0e+3 &   1.4e-1 &  5.0e-1 & 2.0e+6  \\
PCNN     &    8.2e+2 &   1.0e+1 &  1.4e+1 &    1.2e+5  \\
\bottomrule
\end{tabular}
\end{adjustbox}
}{%
    \caption{Predictive Performance and Complexity Metrics}
    \label{bitcoin_table}
}
\end{floatrow}
\end{figure}

Table~\ref{bitcoin_table} shows that the models learning the partitions of the input space, i.e. FFNN-LGT and $\aname$, enhance the predictive performance. Furthermore, the flexibility offered by the paradigm of  Meta-Algorithm~\ref{metaalgo_GET_PCNN} further improves the prediction of the next-day Bitcoin closing price.

We examine the effect of the number of partitions; we repeat the experiment with a fixed number of neurons distributed amongst the subpatterns $\{\hat{f}_n\}_{n=1}^N$ with the test set consisting of the final two weeks of February $20^{th}$ 2018.  We also force the network $\hat{c}$ in~\eqref{eq_deep_zero_sets_definition}, defining all the deep zero-sets $\{\hat{K}_n\}_{n=1}^N$, to scale at a rate of $N^{-1}$ to ensure an even comparison between the FFNNs (i.e.: \aname with a single part) and the genuine $\anames$ with multiple parts.  

\begin{figure}[ht]

\centering
\subfigure[MAE]{\label{fig_Ablation_MAE___Crypto}\includegraphics[width=14em]{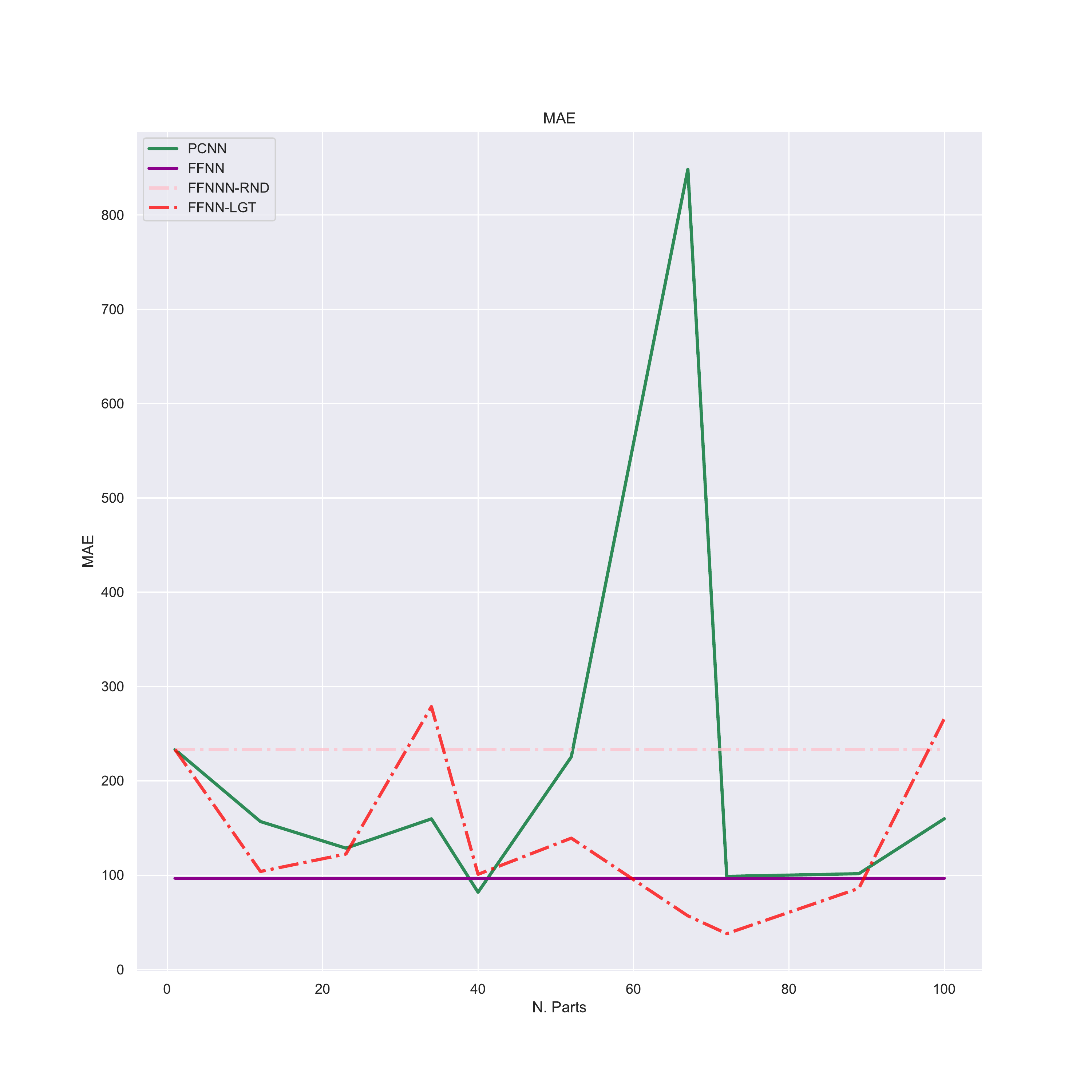}}
\subfigure[P-Time]{\label{fig_Ablation_PTimes___Crypto}\includegraphics[width=14em]{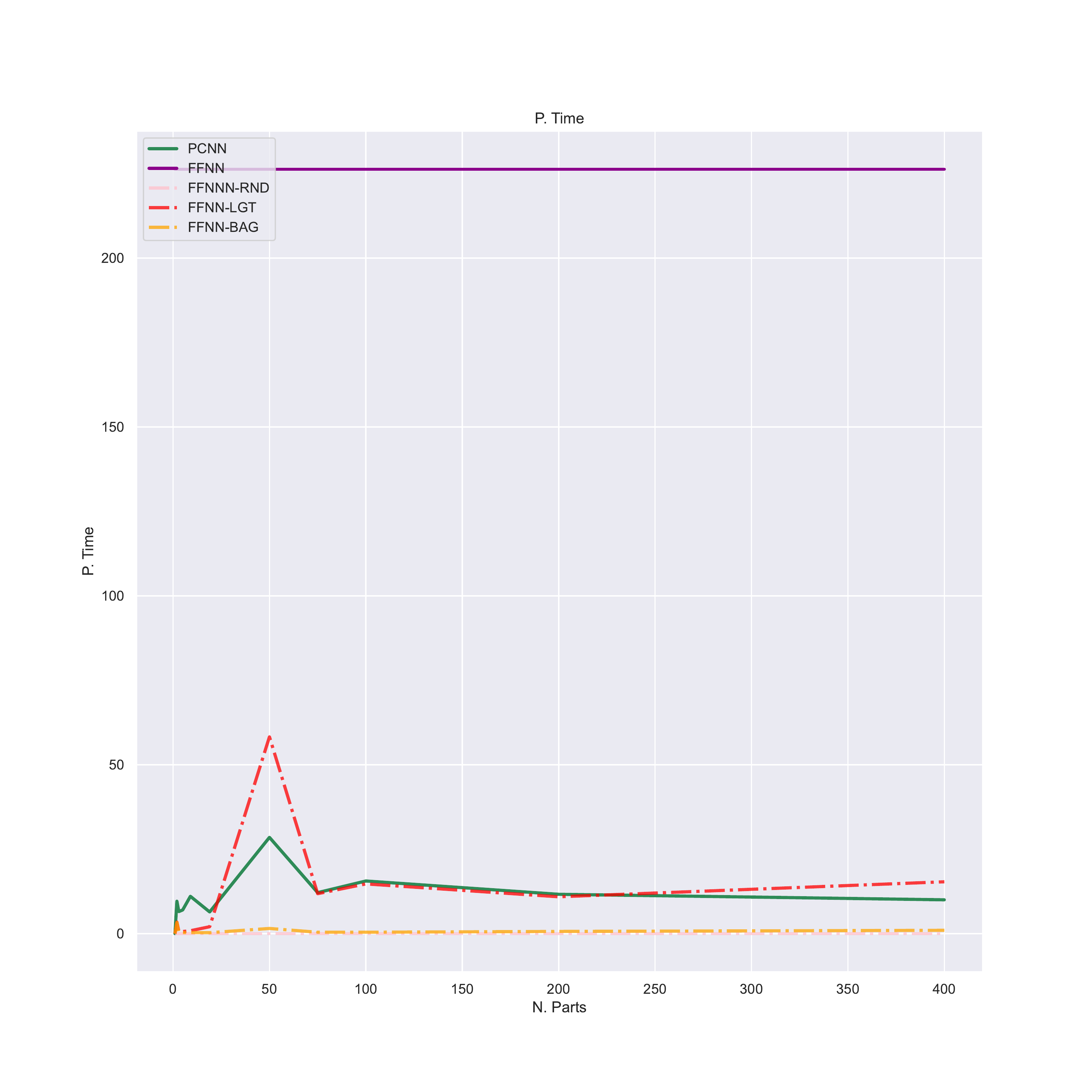}}
\subfigure[Prop. Total Neurons/Input]{\label{fig_Ablation_PropTotNeurons___Crypto}\includegraphics[width=14em]{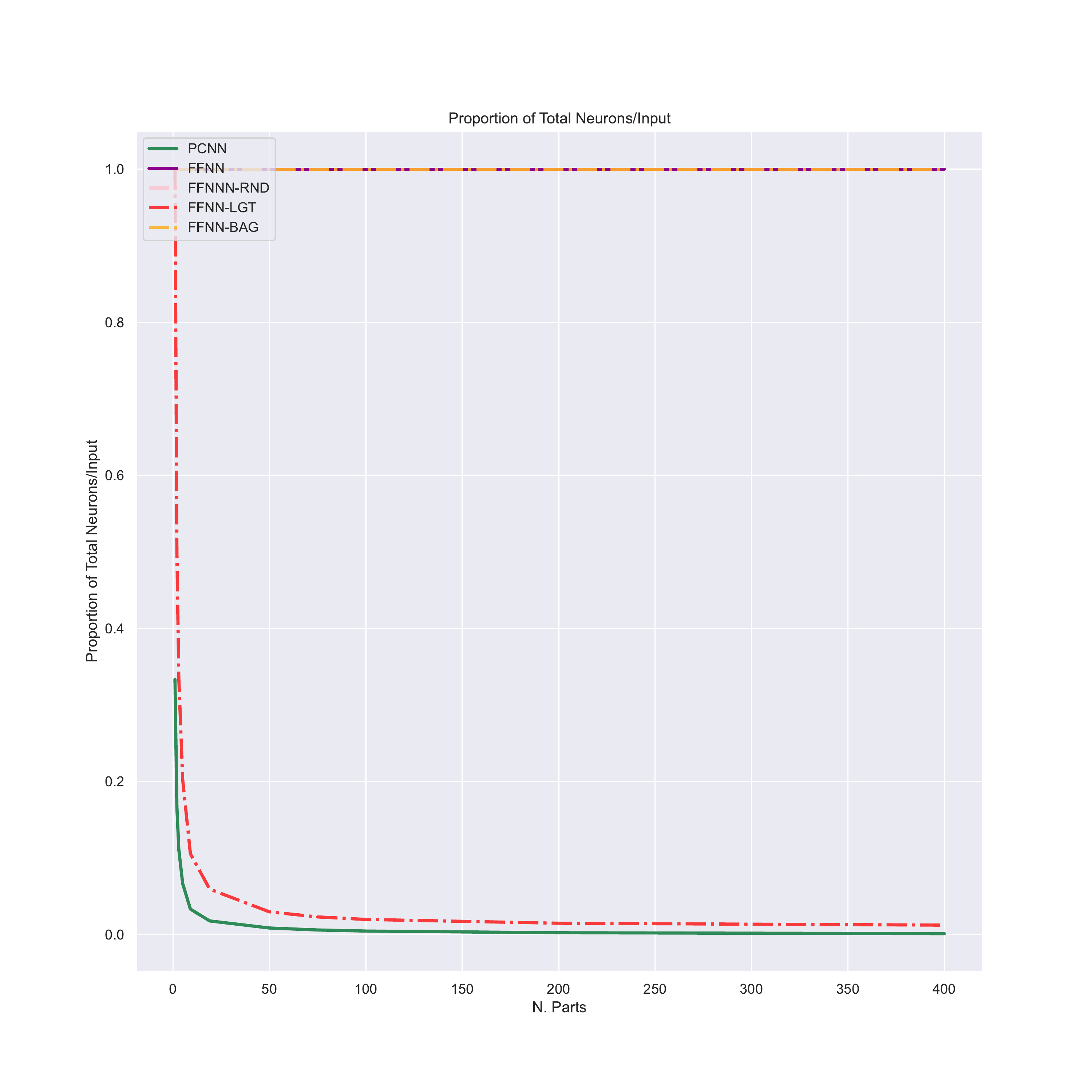}}
\caption{Performance as Function of Number of Parts with Variable Number of Neurons.}
\label{fig_Ablation___Synthetic_Untied_2}
\end{figure}

From Figure~\ref{fig_Ablation_MAE___Crypto}, we see that $\anames$ has a lower test-set MAE than PCNNs with fewer parts, and in particular they outperform PCNNs with 1 part; i.e.: feedforward neural networks.  Figures~\ref{fig_Ablation_MAE___Crypto} and~\ref{fig_Ablation_PropTotNeurons___Crypto} show that the $\anames$ require fewer neurons to produce their predictions.  Moreover, Table~\ref{bitcoin_table} shows that any given input $x\in X$ is processed by far fewer neurons than are available in the entire PCNN; thus any input is first ``triaged'' by the network $\hat{c}$ then assigned to its correct deep zero-set and processes by the correct subpattern of the PCNN specialized to that part of the input space.  

\subsection{Beating an Expert Partition}
\label{s_Implementation_ss_Exp_Part}
Often one has access to a partition of the training inputs derived from expert insight.  This experiment evaluates the impact of using Subroutine~\ref{subroutine_GET_PARTITION} for $\mbox{GET\_PARTITION}$ in Meta-Algorithm~\ref{metaalgo_GET_PCNN} in comparison to using the expert partition.  The experiment will be on the Kaggle housing dataset \cite{KaggleCaliHousePrices} where there is a commonly accepted expert partition \cite{KaggleExpertDiscussion}.  

We compare the $\aname$ model also trained using Meta-Algorithm~\ref{metaalgo_GET_PCNN} and the ADAM optimizer \cite{ADAM2015KingmaB14}  for $\operatorname{GET\_FFNN}$ but now with the expert partition in place of Subroutine~\ref{subroutine_GET_PARTITION}.   This benchmark (PCNN+EXP) quantifies how well the partition initialization of Subroutine~\ref{subroutine_GET_PARTITION} and the partition updating of Meta-Algorithm~\ref{metaalgo_GET_PCNN} (steps 10-12) performs against the commonly accepted partition of that dataset.  The PCNN+EXP benchmark takes the \cite{KaggleExpertDiscussion} partition and then trains an FFNN independently on each part using \cite{ADAM2015KingmaB14} before recombining them with the discontinuous unit (Figure~\ref{fig_PCNNs} in red).  

The impact of the partition updating of Meta-Algorithm~\ref{metaalgo_GET_PCNN} (steps 10-12), given a good initialization of the partition, is quantified by another benchmark (PCNN+EXP+UP).  The PCNN+EXP+UP benchmark is trained just as PCNN+EXP, but it also incorporates the partition updating in Meta-Algorithm~\ref{metaalgo_GET_PCNN} before regrouping the trained FFNNs by the discontinuous unit (Figure~\ref{fig_PCNNs} in red). Thus, PCNN+EXP+UP measures the impact of Meta-Algorithm~\ref{metaalgo_GET_PCNN}'s the partition updating steps.   Since the PCNN+EXP's parts are not implemented by neurons, as are the other benchmark models, reporting $\#Par/x$ in this experiment would not be accurate; rather, Table~\ref{Housing_perform_complexity_table} reports the total number of parameters ($\#$Par) in each model. 

The PCNN+EXP+UP benchmark contrasts against the FFNN-LGT benchmark, also considered in the above experiments, which quantifies the impact of our partitioning procedure against a naively chosen one with no good prior initialization. The benchmarks of the previous experiments are included in Table~\ref{Housing_perform_complexity_table} to gauge the PCNN's performance.  

\begin{table}[ht]
\begin{adjustbox}{width=\columnwidth,center}
\caption{Comparison of $\aname$ with Partitioning and Prediction Benchmarks.}
\label{Housing_perform_complexity_table}
\begin{tabular}{lccrr|lccrr}
	\toprule
\multicolumn{5}{c}{Partition Benchmarks} & \multicolumn{5}{c}{Prediction Benchmarks} \\
\midrule
{} &  MAE  &  L. time &  P. time & $\#$Par
&
&  MAE  &  L. time &  P. time & $\#$Par \\
\midrule
$\aname$+EXP & 3.17e+4 & 8.45e+4 &  4.12e+4 & 8.37e+5 
&
FFNN     &  3.21e+4 &   9.28e+4        & 9.28e+4        & 3.7e+5 \\
$\aname$+EXP+UP  &   3.17e+4 & 1.58e+5 &  6.60e+4 & 1.36e+5 
&
FFNN-BAG &  4.95e+4 &  6.36e+4        & 2.89e+4        & 2.8e+4 \\
$\aname$  &  3.13e+4 &  1.28e+5 &           9.28e+4 &     3.0e+4 
&
FFNN-LGT &  3.18e+4 &  6.37e+4        & 2.90e+4        & 2.8e+4  \\
\bottomrule
\end{tabular}
\end{adjustbox}
\end{table}

The PCNN trained with Meta-Algorithm~\ref{metaalgo_GET_PCNN} out-predicts all the considered benchmarks; in particular, we note the predictive gap of $4k\$ $ MAE from the second-best benchmark (PCNN+EXP+UP) and a gap of $8k\$ $ from the FFNN benchmark.  Thus,  no prior knowledge of the input is needed for a successful PCNN deployment using Meta-Algorithm~\ref{metaalgo_GET_PCNN}.  

When examining the importance of proper partitioning, the gap between the FFNN-LGT and the PCNN models shows that a poorly chosen partition still impacts the model's performance.  In more granularity, the PCNN+EXP has an MAE of $3.17$5e+4 while the PCNN+EXP+UP has a mildly lower MAE of $3.171$e+4. This shows that an accurate choice of the subroutine $\operatorname{GET\_PARTITION}$ is much more impactful in Meta-Algorithm~\ref{metaalgo_GET_PCNN} than the updating steps (10-12).  The reason for this is that the training of PCNN's sub-patterns depends on that initialization. Therefore, a poor choice of an initial partition translates to a reduction of the PCNN's inductive bias.  This is validated by the small gap between the FFNN-LGT, the PCNN+EXP, and PCNN+EXP+UP models.    
\subsection{Ablation within A Controlled Environment}
\label{s_Implementation_ss_synthetic_known_partition}
We begin by ablating the performance of the $\aname$ architecture on various synthetic experiments, wherein we may examine the effect of each component of the synthetic data on the proposed model.  We study the performance of various learning models when faced with the non-linear regression problem
\begin{equation}
    y_n = f(x_n) + \sigma\epsilon_n^{\nu}
    \label{eq_nonlinear_Regression_hard}
    ,
\end{equation}
where for each $n=1,\dots,N$, $x_n$ is sampled uniformly from $[0,1]^d$, $\sigma\geq 0$ is the variance parameter, and are $\epsilon_n^{\nu}$ are i.i.d.  random variables with t-distribution with $\nu$ degrees of freedom.  We vary the behaviour of $f$, the dimensionality $d$, the level of noise $\sigma$, and the size of $\epsilon_n$'s extreme values captured by the heavy-tailedness parameter $\nu$, to understand how the PCNN architecture trained according to Subroutine~\ref{subroutine_GET_PARTITION} behaves.  We consider piecewise continuous $f$ of the form:
\[
f(x) = 
f_1
(Ax)I_{[0,2^{-1}r)}(Ax\bmod r) + 
f_2(Ax)I_{[2^{-1}r,r)}((Ax\bmod r)),
\]
where $r>0$ captures the rate at which the sub-patterns $f_1$ and $f_2$ interchange and  $A$ is a randomly generated $D\times 1$-matrix with i.i.d. standard Gaussian entries.  For instance, if $d=100$ then there is on average, $10$ discontinuities in each direction of the input space.  The difficulties in the non-linear regression problem~\eqref{eq_nonlinear_Regression_hard} arise from the many discontinuities of $f$, the opposing and oscillating trends of the $f_i$, the problem's dimensionality, and the heavy-tailedness of its noise.  

The previous experiment's predictive results were based on concrete dollar values; however, the outcomes of these experiments are just numerical values.  Therefore to maintain interpretability, all performance metrics will be reported as a fraction over the principle benchmark, i.e. the FFNN model's performance metric.   

Each experiment also reports the mean total number of neurons processing each of the network's inputs ($\#$Par/x), the total number of parts used in building each neural model, and the quantities $d$, $\sigma$, the number of data points, $\nu$, and $r$.  To frame the irregularity of each function being learned, each table will be accompanied by a plot of the samples from the noiseless target function $f$ (in {\color{red}{red}}) and the noisy training data (in {\color{blue}{blue}}) in the ``visualizable case'' \textit{where $d=1$ and $A=1$}.
\subsubsection{Parsing oscillations from Discontinues}
\label{s_Experiments_ss_Synthetic_sss_Jumps}

This experiment examines how well $\anames$ can learn discontinuities in the presence of distinct oscillating subpatterns.  For this, we set $f_1(u)=1+e^u\cos(u)$ and $f_2(u)=-1-u^2\cos(u)$ and vary $r$.  

\begin{figure}[H]
\centering
\begin{floatrow}
\ffigbox{%
\begin{adjustbox}{width=\columnwidth,center}
  \centerline{\includegraphics[height=0.2\textwidth]{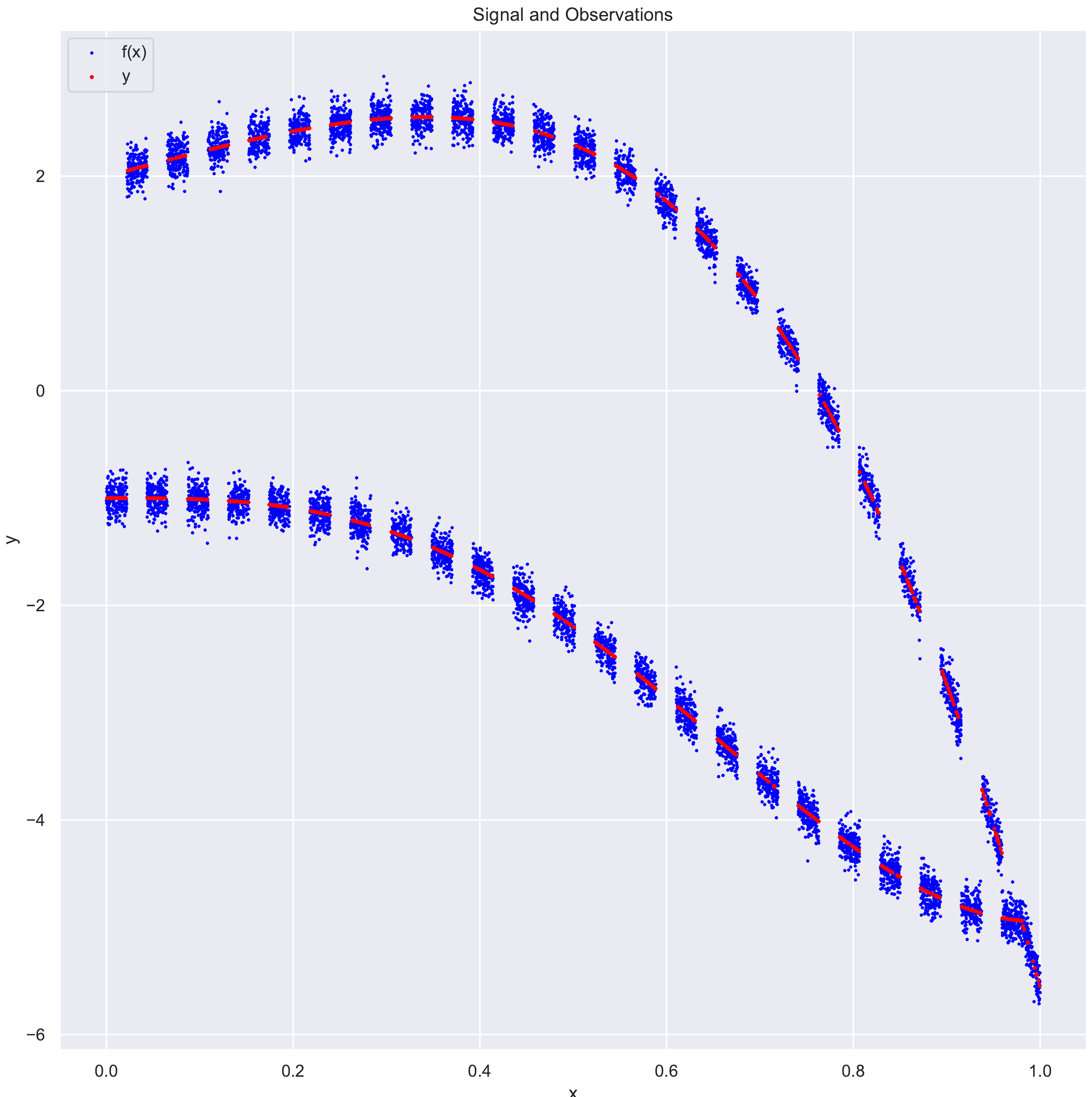}}
\end{adjustbox}
}{%
\caption{{
    \footnotesize{
    {\color{red}{Test}} and Noisy {\color{blue}{Training}} Data  (if $d=1=A$).
    }
    }}%
  \label{fig_Synthetic_Noise_1D_2}
}
\capbtabbox{%
\centering
\begin{adjustbox}{width=\columnwidth,center}
\begin{tabular}{lrrlrrr|rrrrr}
\toprule
{} & MAE &    P. Time & L. Time &      $\#$Par/x & $\#$Parts &  d &   $\sigma$ &      $\#$Data &  $\nu$ & $r$ \\
\midrule
FFNN     & 1.00e+00 & 1.00e+00 &                   - &       1.00e+00 &         1 &  100 &  0.01 &  10000 &     30 & 0.25 \\
FFNN-RND & 1.00e+04 & 1.45e-04 &                   - &       2.49e-03 &         1 &  100 &  0.01 &  10000 &     30 & 0.25 \\
FFNN-BAG & 2.96e+00 & 1.05e-02 &  1.77 &       9.95e-01 &       400 &  100 &  0.01 &  10000 &     30 & 0.25 \\
FFNN-LGT & 1.91e+00 & 1.11e-01 &  1.87 &       1.00e+00 &       400 &  100 &  0.01 &  10000 &     30 & 0.25 \\
PCNN     & 9.97e-01 & 1.93e-01 &  1.95 &       1.99e+00 &       400 &  100 &  0.01 &  10000 &     30 & 0.25 \\
\midrule
\midrule
FFNN     & 1.00e+00 & 1.00e+00 &                    - &       1.00e+00 &         1 &  10e+2 &  0.01 &  10e+5 &     30 & 0.1 \\
FFNN-RND & 9.06e+03 & 5.35e-05 &                    - &       2.49e-03 &         1 &  10e+2 &  0.01 &  10e+5 &     30 & 0.1 \\
FFNN-BAG & 5.72e+00 & 3.46e-03 &  1.58e-1 &       3.66e-01 &       147 &  10e+2 &  0.01 &  10e+5 &     30 & 0.1 \\
FFNN-LGT & 5.98e+00 & 4.30e-02 &  1.98e-1 &       3.67e-01 &       147 &  10e+2 &  0.01 &  10e+5 &     30 & 0.1 \\
PCNN     & 9.87e-01 & 7.70e-02 &  2.31e-1 &       7.31e-01 &       147 &  10e+2 &  0.01 &  10e+5 &     30 & 0.1 \\
\bottomrule
\end{tabular}
\end{adjustbox}
}{%
    \caption{{
    \footnotesize{Performance Metrics/FFNN - Varying ``Discontinuity Rate'' $(r)$.}
    }}%
    \label{tab__Sythetic__Fix_Neurons_QTrue_Varying_noiselevel}
}
\end{floatrow}
\end{figure}

As $r$ approaches $0$, $f$ contains more regions of discontinuities.  Table~\ref{tab__Sythetic__Fix_Neurons_QTrue_Varying_R} validates our hypothesis by showing that for small $r$ the feedforward networks (FFNN and FFNN-RND) have trouble capturing these discontinuities.  In contrast, even with a mostly random training procedure and subpatterns generated by relatively narrow layers, the PCNN is expressive enough to bypass these issues.  

\subsubsection{Learning from Noisy Data}
\label{s_Experiments_ss_Synthetic_sss_Noise}
Next, we examine the $\aname$'s predictive performance when there is variable levels of noise.  In this experiment, we move between the low and high ``signal-to-noise ration'' regimes by progressively increasing the variance parameter $\sigma$ and reducing the parameter $\nu$, which increases the size and frequency of extreme values of the $\epsilon_x^{\nu}$ \citep{ExtremeValueTheory_HaanLaurensFerreira2006}.  In this experiment, we consider an oscillatory pattern $f_1(u)=1 + \sin(10u)$ which is difficult to parse from noisy data due to its oscillations and a stable pattern $f_2(u)=-2 - u^2$.

\begin{figure}[ht]
\centering
\begin{floatrow}
\ffigbox{%
\begin{adjustbox}{width=\columnwidth,center}
  \centerline{\includegraphics[height=0.2\textwidth]{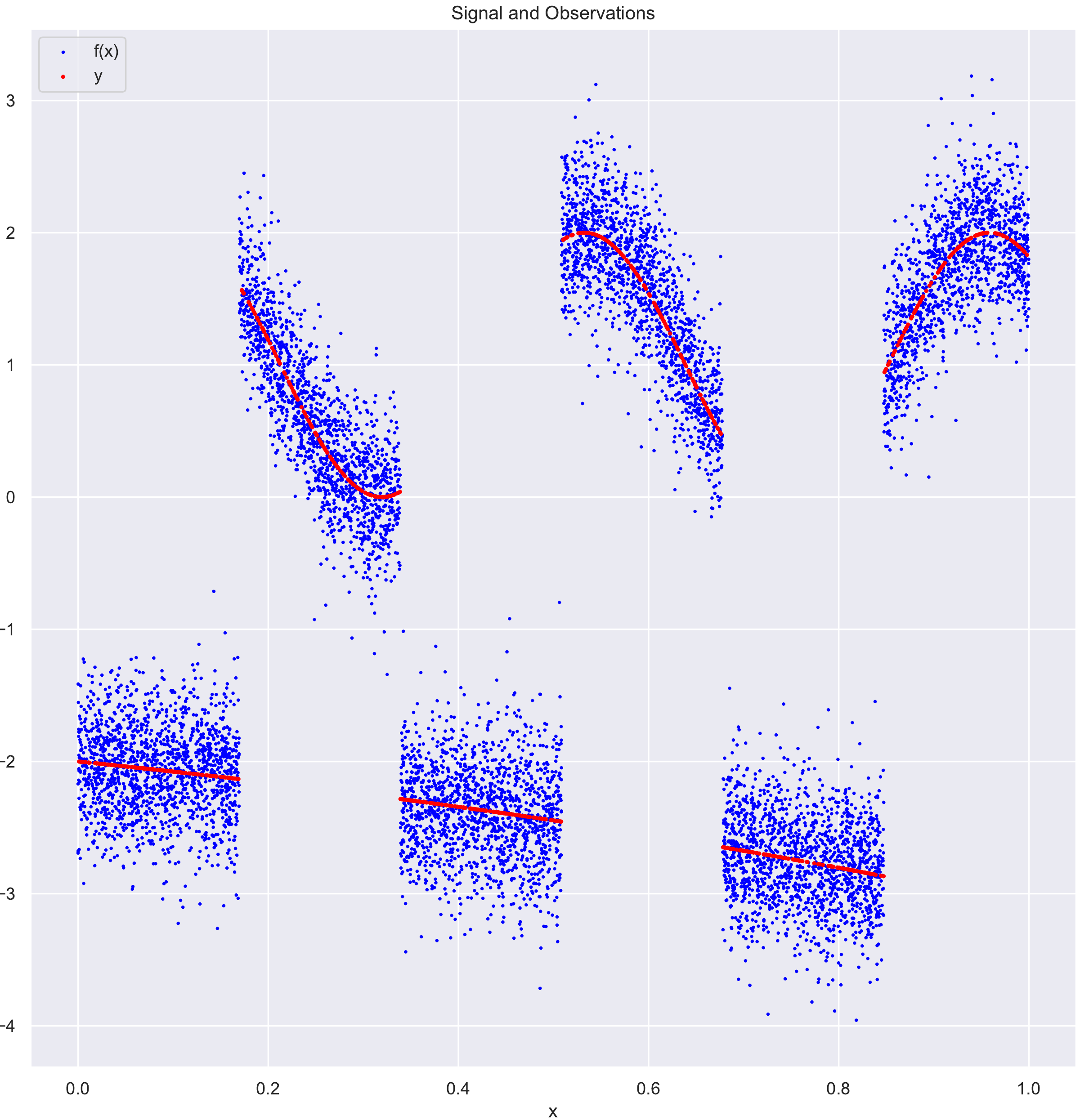}}
\end{adjustbox}
}{%
  \caption{{
    \footnotesize{
    {\color{red}{Test}} and Noisy {\color{blue}{Training}} Data  (if $d=1=A$).
    }
    }}%
  \label{fig_Synthetic_Noise_1D_3}
}
\capbtabbox{%
\begin{adjustbox}{width=\columnwidth,center}
\centering
\begin{tabular}{lrrlrr|rrrrr}
\toprule
{} & MAE &    P. Time & L. Time &       $\#$Par/x & $\#$Parts &  d &   $\sigma$ &      $\#$Data &  $\nu$ & $r$ \\
\midrule
FFNN     & 1.00e+00 & 1.00e+00 &                   - &       1.00e+00 &          1 &  10e+2 &  0.01 &  10e+5 &    15 & 0.25 \\
FFNN-RND & 1.69e+01 & 5.53e-05 &                   - &       2.49e-03 &           1 &  10e+2 &  0.01 &  10e+5 &    15 & 0.25 \\
FFNN-BAG & 1.18e+00 & 3.48e-02 &  21.74 &       3.70e+00 &       1400 &  10e+2 &  0.01 &  10e+5 &    15 & 0.25 \\
FFNN-LGT & 1.11e+00 & 3.82e-01 &  22.09 &       3.72e+00 &        1400 &  10e+2 &  0.01 &  10e+5 &    15 & 0.25 \\
PCNN     & 8.86e-01 & 1.33e-01 &  21.84 &       7.40e+00 &        1400 &  10e+2 &  0.01 &  10e+5 &    15 & 0.25 \\
\midrule
FFNN     & 1.00e+00 & 1.00e+00 &                   - &       1.00e+00 &         1 & 10e+2 &  0.1 &  10e+5 &    5 & 0.25 \\
FFNN-RND & 1.70e+01 & 7.14e-05 &                   - &       2.49e-03 &         1 & 10e+2 &  0.1 &  10e+5 &    5 & 0.25 \\
FFNN-BAG & 1.21e+00 & 3.19e-02 &  27.73 &       4.98e+00 &      2000 & 10e+2 &  0.1 &  10e+5 &    5 & 0.25 \\
FFNN-LGT & 1.19e+00 & 2.14e-01 &   27.91 &       5.00e+00 &      2000 & 10e+2 &  0.1 &  10e+5 &    5 & 0.25 \\
PCNN     & 8.48e-01 & 8.92e-02 &  27.78 &       9.95e+00 &      2000 & 10e+2 &  0.1 &  10e+5 &    5 & 0.25 \\
\bottomrule
\end{tabular}
\end{adjustbox}
}{%
 \caption{{
    \footnotesize{Performance Metrics/FFNN - Varying Signal-to-Noise Ratio $(\sigma,\nu)$.}
    }}%
\label{tab__Sythetic__Fix_Neurons_QTrue_Varying_R}
}
\end{floatrow}
\end{figure}

Table~\ref{tab__Sythetic__Fix_Neurons_QTrue_Varying_noiselevel} shows that the $\aname$ model is capable of producing reliable results even when approximating complicated functions in the presence of a high signal-to-noise ratio.  We find that the subpatterns $\hat{f}_n$ selected when noise is high tend to be narrow and the number of parts tends to be larger.  Heuristically, this means that the number of parts selected for $\aname$ tends to be large as the signal-to-noise ratio lowers and visa-versa as the signal-to-noise ratio increases.  

\subsubsection{Learning From Few Training Samples}
\label{s_Experiments_ss_Synthetic_sss_Sample_Size}
We examine the impact of small sample size on the $\aname$ model's performance.  In this experiment, we fix a relatively simple discontinuous function $f_1(u)=x$ and $f_2(u)=x^2$, and vary the size of the training dataset $(\#\mbox{Data})$.   Figure~\ref{tab__Sythetic__Fix_Neurons_QTrue_Varying_N_TraininSamples} shows that, just as with the other experiments which all used $10e+5$ training instances, when the $\#\mbox{Data}$ is reduced the $\aname$ model's predictive performance remains comparable to that of FFNN model.  

\begin{figure}[ht]
\centering
\begin{floatrow}
\ffigbox{%
\begin{adjustbox}{width=\columnwidth,center}
  \centerline{\includegraphics[height=0.2\textwidth]{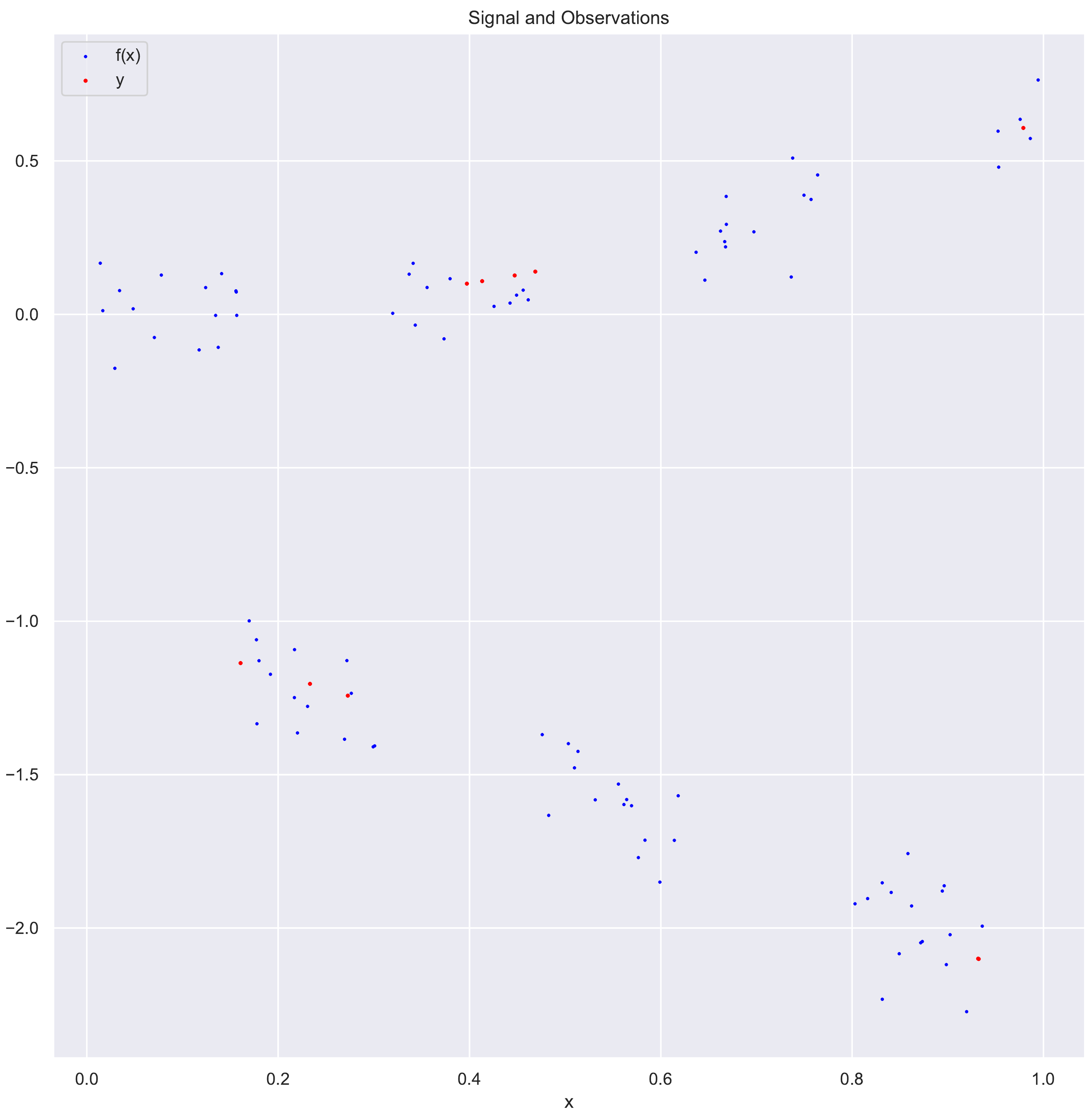}}
\end{adjustbox}
}{%
  \caption{{
    \footnotesize{
    {\color{red}{Test}} and Noisy {\color{blue}{Training}} Data  (if $d=1=A$).
    }
    }}%
  \label{fig_Synthetic_Noise_1D_1}
}
\capbtabbox{%
\begin{adjustbox}{width=\columnwidth,center}
\centering
\begin{tabular}{lrrlrr|rrrrr}
\toprule
{} & MAE &    P. Time & L. Time &     $\#$Par/x &   $\#$Parts &  d &   $\sigma$ &      $\#$Data &  $\nu$ & $r$ \\
\midrule
FFNN     & 1.00e+00 & 1.00e+00 &                       - &       1.00e+00 &           1 &  1 &  0.01 &  10e+2 &     30 & 0.25 \\
FFNN-RND & 1.28e+00 & 4.80e-04 &                       - &       4.98e-03 &           1 &  1 &  0.01 &  10e+2 &     30 & 0.25 \\
FFNN-BAG & 8.81e-01 & 2.39e-04 &  2.38e-4 &       4.98e-03 &           1 &  1 &  0.01 &  10e+2 &     30 & 0.25 \\
FFNN-LGT & 8.81e-01 & 1.02e-04 &  1.02e-4 &       4.99e-03 &           1 &  1 &  0.01 &  10e+2 &     30 & 0.25 \\
PCNN     & 8.81e-01 & 4.80e-04 &   4.80e-04 &       9.95e-03 &         1 &  1 &  0.01 &  10e+2 &     30 & 0.25 \\
\midrule
FFNN     & 1.00e+00 & 1.00e+00 &                  - &       1.00e+00 &          1 &  1 &  0.01 &  1000 &     30 & 0.25 \\
FFNN-RND & 1.19e+06 & 1.40e-04 &                  - &       6.23e-04 &           1 &  1 &  0.01 &  1000 &     30 & 0.25 \\
FFNN-BAG & 3.54e+11 & 1.64e-02 &  4.02 &       4.99e-01 &         1 &  1 &  0.01 &  1000 &     30 & 0.25 \\
FFNN-LGT & 1.28e+00 & 0.01 &  4.01 &       5.01e-01 &         800 &  1 &  0.01 &  1000 &     30 & 0.25 \\
PCNN     & 1.02e+00 & 1.09e-01 &  4.11 &       9.98e-01 &         800 &  1 &  0.01 &  1000 &     30 & 0.25 \\
\bottomrule
\end{tabular}
\end{adjustbox}
}{%
\caption{{
    \footnotesize{Performance Metrics/FFNN - Varying $\#$Train. Data. $(\#\mbox{Data})$.}
    }}%
\label{tab__Sythetic__Fix_Neurons_QTrue_Varying_N_TraininSamples}
}
\end{floatrow}
\end{figure}

Table~\ref{tab__Sythetic__Fix_Neurons_QTrue_Varying_N_TraininSamples} shows when few training data points are available then fewer parts are required, because there is enough space for a classical FFNN (i.e.: a $\aname$ with a single part) to interpolate the data even if $f$ is discontinuous.  Nevertheless, as the size of the training dataset increases the FFNN does not have as much flexibility to meander through the training data as the ${\aname}$ model does, as is seen both in Table~\ref{tab__Sythetic__Fix_Neurons_QTrue_Varying_N_TraininSamples} and all the previous experiments.  

We see that the experiments which used a stochastic optimization method for the subroutine $\mbox{GET\_FFNN}$ tended to use much fewer partitions than the experiments which randomized the involved FFNN's hidden layers and only trained their final layer.  This validates the theoretical results derived in Theorems~\ref{thrm_improvement_largeness}-\ref{thrm_UAT_PC} guaranteeing that the PCNNs are more expressive than their FFNN counterparts.

\section{Conclusion}\label{s_Conclusion}
We introduced a new deep neural model capable of uniformly approximating a large class of discontinuous functions.  We provided theoretical universal approximation guarantees for each of the $\aname$'s parts and the deep neural models as a whole.  We showed that the $\aname$'s sub-pattern structure could be exploited to decouple the model's training procedure, thus, avoided passing gradient updates through the model's non-differentiable $I_{(\gamma,1]}$ unit.  In this way, the $\aname$ offers the best of both worlds.  It is as expressive as deep neural networks with discontinuous activation functions while simultaneously being trainable with (stochastic) gradient-descent algorithms like feed-forward networks with differentiable activation functions.

\appendix
\section{Proof of the Negative Result for FFNNs}
The first section in our appendix contains the proofs of results illustrating the limitations of FFNNs in approximating piecewise continuous functions.  
\begin{proof}[{Proof of Proposition~\ref{prop_Disc_Gap}}]
By definition (i.e.,~\eqref{defn_metric_step2}) we have that $\inf_{R(G)=g}\,|N(f)-G|\leq D_{PC}(f|\hat{f})$.  Now, since $1< N(g)$, then 
\begin{equation}
    |N(f)-N(g)|\leq \inf_{R(G)=g}\,|N(f)-G| \leq D_{PC}(f|\hat{f})
    \label{eq_PC_divergence_lower_bound}
    .
\end{equation}
Since $\sigma \in C(\rr)$, then $\NN[d,D]\subseteq C(X,\rrD)$ and since $X$ is compact then for each $\hat{f}\in \NN[d,D]$, $(\hat{f},I_{X})$ is a representation of $\hat{f}$.  Hence, $N(f)=1$ and therefore:
\begin{equation}
    1\leq N(f)-1=|N(f)-N(g)|
    \label{eq_PC_divergence_lower_bound_2}
    .
\end{equation}
Plugging~\eqref{eq_PC_divergence_lower_bound_2} into~\eqref{eq_PC_divergence_lower_bound} yields the desired lower-bound on $D_{PC}(f,|\hat{f})$ for each $\hat{f}\in \NN[d,D]$.  Lastly, Example~\ref{ex_gap_cnt_disc2} gives the last claim.
\end{proof}
Next we turn our attention to the proof of Proposition~\ref{prop_non_universal_FFNN_cnt}.  Proposition~\ref{prop_non_universal_FFNN_cnt} implied by, and is a special case of, the following general result.  
\begin{ass}[Regularity of Partition I]\label{ass_appendix_badness} 
The partition $\left\{ K_n\right\}_{n=1}^{N}$ satisfies the following:
\begin{enumerate}[(i)]
	\item \textbf{Non-Triviality:} For each $n,m\leq N$, $\operatorname{int}(K_n)\neq \emptyset$ and if $n\neq m$ then $\operatorname{int}(K_n)\cap \operatorname{int}(K_m)=\emptyset$.  
	\item \textbf{Pairwise Connectedness:} For every $n\leq N$ there is some $\tilde{n}<N$ with $\tilde{n}\neq n$ and $K_n\cap K_{\tilde{n}}\neq \emptyset$,
			\item \textbf{Regularity: } For $n\leq N$ we have $\overline{\intt[K_n]}=K_n$.  
		\end{enumerate}
	\end{ass}
	\begin{prop}\label{prop_non_universal_FFNN_cnt_appendix_version_extended}
	Let $\{K_n\}_{n=1}^N$ be a partition of $X$ satisfying Assumption~\ref{ass_appendix_badness} and let $\sigma \in C(\rr)$.  For every $\delta>0$, there exists some $f_{\delta}$ of the form~\eqref{eq_sub_patterns} with $\operatorname{supp}(f_{\delta})\subseteq K_1\cup K_n$, for some $1<n\leq N$, such that 
		$$
			\inf_{\underset{\operatorname{supp}(f)\subseteq K_1\cup K_n}{f \in \NN}} \sup_{x \in \bigcup_{k=1}^{N} K_k} \, 
			\left\|
			f(x)
			-
			f_{\delta}(x)
			\right\|\geq \delta.
		$$
	\end{prop}
	In what follows, for any $d\in \nn_+$, $x \in \rrd$, and $\delta>0$ we use $\operatorname{Ball}(x,\delta)$ to denote the set $\operatorname{Ball}(x,\delta)\triangleq \left\{z\in\rrd:\,\|x-z\|<\delta\right\}$.  
		\begin{proof}[{Proof of Proposition~\ref{prop_non_universal_FFNN_cnt_appendix_version_extended}}]
			By Assumption~\ref{ass_appendix_badness} (ii), there exists some $1<n<N$ such that $K_1\cap K_n\neq \emptyset$ and by Assumption~\ref{ass_appendix_badness} (i) we have that $\operatorname{int}(K_1)\cap \operatorname{int}(K_n)=\emptyset$.  Let $f_{\delta}$ be defined by
			$$
			f_{\delta} = -2c_1 I_{\operatorname{int}(K_1)} + 2c_n I_{\operatorname{int}(K_n)},
			$$
			where $c_1,c_n \in \rrD$ are such that $\|c_1-c_n\|=\delta$.  
			Thus, for every $f \in \NN$ we compute
			\begin{equation}
			\sup_{x \in K} \left\|f(x)-f_{\delta}(x)\right\|
			=
			\max_{i=1,n}
			\sup_{x \in \operatorname{int}(K_i)} \left\|f(x)-f_{\delta}(x)\right\|
			\geq 
			\max_{i=1,n}
			\sup_{x \in \operatorname{int}(K_i)} \left\|f(x)-c_i\right\|
			\label{proof_prop_non_universal_FFNN_cnt}
			.
			\end{equation}
			Let $f \in \NN$, then there exist composable affine functions $W_J,\dots,W_1$ such that $f=W_J\circ \sigma\bullet \dots \circ \sigma \bullet W_1$.  Since $\sigma$ is continuous, each $W_j$ is continuous, and since the composition of continuous functions is again continuous then $f$ is continuous.  Since $f$ is continuous, then its supremum on any bounded open subset $B\subseteq \rrd$ is equal to the maximum of $f$ over the closure of $B$ in $\rrd$; hence we refine~\eqref{proof_prop_non_universal_FFNN_cnt} to
			\begin{equation}
			\sup_{x \in K} \left\|f(x)-f_{\delta}(x)\right\|
			=
			\max_{i=1,n}
			\sup_{x \in \operatorname{int}(K_i)} \left\|f(x)-c_i\right\|
			=
			\max_{i=1,n}
			\max_{x \in K_i} \left\|f(x)-c_i\right\|
			\label{proof_prop_non_universal_FFNN_cnt_2}
			.
			\end{equation}
			Suppose that $\sup_{x \in K} \|f(x)-f_{\delta}(x)\|<\delta$ for some $f \in \NN$.  Then~\eqref{proof_prop_non_universal_FFNN_cnt_2} implies that $f(x)\in B_{\delta}(c_i)$ for $i=1,n$ but by definition of $c_1$ and $c_n$ we had that $\|c_1-c_n\|=\delta$ and therefore $B_{\delta}(c_1)\cap B_{\delta}(c_n)=\emptyset$, a contradiction.  Hence, there does not exist some $f \in \NN$ satisfying $\sup_{x \in K} \|f(x)-f_{\delta}(x)\|<\delta$.  
		\end{proof}
\section{Proof of Supporting Results}\label{a_PROOFSSUPPORTING}
\begin{proof}[{Proof of Lemma~\ref{lem_boundedness_of_Architopes}}]
If $f \in \PCNN{\fff}$, then by definition there exist $f_1,\dots,f_N \in \fff$ and some $K_1,\dots,K_N\subseteq \xxx$, such that $f=\sum_{n=1}^N f_n I_{K_n}$.  Since $\fff\subseteq C(X,\rrD)$, since $X$ is compact, and since every continuous function achieves its maximum on a compact space then, for every $n=1,\dots, N$, $\sup_{x\in X}\|f_n(x)\|<\infty$.  Therefore, we compute
$$
\begin{aligned}
\sup_{x \in X} \left\|
\sum_{n=1}^N f_n(x) I_{K_n}(x)
\right\| 
\leq & 
\sup_{x \in X} \sum_{n=1}^N \left\|
f_n(x) I_{K_n}(x)
\right\| \\
= & \sup_{x \in X} \sum_{n=1}^N I_{K_n}(x)\left\|
f_n(x) 
\right\| \\
\leq & \sup_{x \in X} \sum_{n=1}^N 1\left\|
f_n(x) 
\right\| <\infty.
\end{aligned}
$$
Thus, $f$ is bounded on $X$ if $f \in \PCNN{\fff}$.  
\end{proof}
\begin{proof}[{Proof of Proposition~\ref{prop_existence}}]
If $f\in \mathcal{B}(X,\rrD)$  admits a representation of the form $\sum_{n=1}^{N'} f_nI_{K_n}$ for some ${N'}\in \nn_+$, some $f_1,\dots,f_{N'}\in C(X,\rrD)$ and some $K_1,\dots,K_{N'}\in \text{Comp}(X)$ then, 
$$
N(f)= \inf\{N\in :\, \exists ((f_n,K_n))_{n=1}^N\in 
\cup_{k\in \nn_+,\, k\leq {N'}}\, [C(X,\rrD)\times \text{Comp}(X)]^k:\, f=\rrr\left(((f_n,K_n))_{n=1}^N\right)
\}.
$$
Since the set $\{1,\dots,N'\}$ is finite, then it admits a minimum and therefore $N(f)\in \nn_+$ exists.  
\end{proof}

\begin{proof}[{Proof of Proposition~\ref{prop_decoupling_lemma}}]
By definition, $N(f)=\# (\hat{f}_n,\hat{K}_n))_{n=1}^{N(f)}$.  Therefore,~\eqref{defn_metric_step2} is upper-bounded using $d_{\text{Step 1}}$ of~\eqref{defn_metric_step1} via:
\begin{align}
\allowdisplaybreaks
\nonumber
    D_{PC}\left(f|R\left(
((\hat{f}_n,\hat{K}_n))_{n=1}^{N(f)}
\right)\right) \leq &
\inf_{R(((\tilde{f}_n,\tilde{K}_n))_{n=1}^{N})=f,\, N = N(f)}
        d_{\text{Step 1}}\left(
        ((\tilde{f}_n,\tilde{K}_n))_{n=1}^{N}
            ,
        ((\hat{f}_n,\hat{K}_n))_{n=1}^{N}
    \right)
    \\
    \label{eq_proposition_prop_decoupling_lemma_non_infinity_observation}
    = &
    \inf_{R(((\tilde{f}_n,\tilde{K}_n))_{n=1}^{N})=f,\, N = N(f)}
    \max_{1\leq n\leq N}\max\left\{
        \|\tilde{f}_n-\hat{f}_n\|_{\infty}
            ,
        d_H(\tilde{K}_n,\hat{K}_n)
    \right\}
    \\
    \leq &
    \max_{1\leq n\leq N(f)}\max\left\{
        \|f_n-\hat{f}_n\|_{\infty}
            ,
        d_H(K_n,\hat{K}_n)
    \right\}
    \label{eq_proposition_prop_decoupling_lemma}
    .
\end{align}
Note that, we have use the fact that $f \in PC(X,\rrD)$ to conclude that $N(f)<\infty$ and therefore the infimum in~\eqref{eq_proposition_prop_decoupling_lemma_non_infinity_observation} is not vacuously $\infty$ and we have used the fact that $\#((\hat{f},\hat{K}_n))_{n=1}^{N(f)}=N(f)$ to conclude that each $d_{\text{Step 1}}\left(
        ((\tilde{f}_n,\tilde{K}_n))_{n=1}^{N}
            ,
        ((\hat{f}_n,\hat{K}_n))_{n=1}^{N}
    \right)<\infty 
$ in~\eqref{eq_proposition_prop_decoupling_lemma_non_infinity_observation} is finite.  Next, we observes that for each $N\in \nn_+$ and each $x \in \rrflex{N}$ it follows that $\max_{1\leq n\leq N}|x_n|\leq \sum_{n=1}^{N}|x_n|$; thus, we upper-bound on the right-hand side of~\eqref{eq_proposition_prop_decoupling_lemma} as follows:
\begin{align}
\allowdisplaybreaks
\nonumber
    D_{PC}\left(f|R\left(
((\hat{f}_n,\hat{K}_n))_{n=1}^{{N(f)}}
\right)\right) \leq &
    \max_{1\leq n\leq {N(f)}}\max\left\{
        \|f_n-\hat{f}_n\|_{\infty}
            ,
        d_H(K_n,\hat{K}_n)
    \right\}\\
    \nonumber
       \leq & 
    \sum_{n=1}^{N(f)}
    \max\left\{
        \|f_n-\hat{f}_n\|_{\infty}
          ,
        d_H(K_n,\hat{K}_n)
    \right\}
    \\
    \leq & 
    \sum_{n=1}^{N(f)}
    \left(
        \|f_n-\hat{f}_n\|_{\infty}
          +
        d_H(K_n,\hat{K}_n)
    \right)
    \label{eq_proposition_prop_decoupling_lemma_2}
    .
\end{align}
Thus,~\eqref{eq_proposition_prop_decoupling_lemma_2} implies that the estimate~\eqref{eq_prop_decoupling_lemma} holds.  
\end{proof}
\section{Proofs of Main Results}\label{s_proofs_main_results}\label{a_PROOFSMAIN}
This appendix contains proofs of the paper's main results.   We draw the reader’s attention to the fact that many of the paper’s results are interdependent and, as such, this appendix organizes our results' proofs in their logical order and that this order may differ from the one used in the paper exposition.  
\subsection{Proofs concerning the closure of the $\anames$ in $\mathcal{B}(X,\rr^D)$}\label{ss_Appendix_proof_Theorem_Expressive}
\begin{proof}[Proof and Discussion of Theorem~\ref{thrm_improvement}]
Define the subset $\mathcal{V}
\subseteq \mathcal{B}(X,\rrD)$ by, $f \in \mathcal{V}$ 
if and only if there exist some $f_1,f_2 \in C(X,\rrD)$ for which
\begin{equation}
    f = (I_{(\frac{1}{2},1]}\circ \sigma_{\operatorname{sigmoid}}\circ g)f_1 + f_2
    \label{eq_representation_quick_and_easy}
    ;
\end{equation}
where $g \in \ggg$ is the same as in part (i).  Let $f,\tilde{f}\in \mathcal{V}$ and $k\in \rr$.  Then, 
$$
f+k\tilde{f} = 
(I_{(\frac{1}{2},1]}\circ \sigma_{\operatorname{sigmoid}}\circ g)(f_1+k\tilde{f}_1) + (f_2+k\tilde{f}_2) \in \mathcal{V};
$$
where, mutatis mutandis, we have represented $\tilde{f}$ in the form~\eqref{eq_representation_quick_and_easy}.
Hence, $\mathcal{V}$ is a vector space containing $\overline{\fff}$ as a proper subspace; where the latter claim follows by taking $f_2 =0$ in the representation~\eqref{eq_representation_quick_and_easy} and $f_1\in \overline{\fff}=C(X,\rrD)$ arbitrary.  
In particular, by construction, $\mathcal{V}$ is a subspace of $\mathcal{B}(X,\rrD)$.  By the Uniform Limit Theorem \citep[Theorem 21.6]{munkres2014topology} $C(X,\rrD)$ is closed with respect to $d_{\infty}$; hence, $C(X,\rrD)$ a closed proper subset of $\mathcal{V}$.  Hence, it is not dense therein.  Applying \citep[Excersize 11.4.3 (f)]{narici2010topological} we conclude that $C(X,\rrD)$ is nowhere dense in $\mathcal{V}$.  It is therefore sufficient to show that $\mathcal{V}\subseteq \overline{\PCNN{\fff}}$ to conclude that $C(X,\rrD)$ is nowhere dense $\overline{\PCNN{\fff}}$, and therefore, $\fff$ is nowhere dense in $\overline{\PCNN{\fff}}$.  

Let $f \in \mathcal{V}$, which we represent according to~\ref{eq_representation_quick_and_easy}, and let $\epsilon>0$.  By the hypothesized density of $\fff$ in $C(X,\rrD)$ there exist some $\hat{f}^{\epsilon}_1,\hat{f}^{\epsilon}_2\in \fff$ satisfying 
\begin{equation}
    \max_{i=1,2}\,\|f_i-\hat{f}^{\epsilon}_i\|_{\infty}<\frac{\epsilon}{2}
    \label{eq_definition_estimate}
    .
\end{equation}
Therefore, from~\eqref{eq_definition_estimate} and~\eqref{eq_representation_quick_and_easy} we compute the estimate
$$
\begin{aligned}
\left\|
f - 
\left(
(I_{(\frac{1}{2},1]}\circ \sigma_{\operatorname{sigmoid}}\circ g)\hat{f}_1^{\epsilon} + \hat{f}_2^{\epsilon}
\right)
\right\|_{\infty}
= &
\left\|
(I_{(\frac{1}{2},1]}\circ \sigma_{\operatorname{sigmoid}}\circ g)(f_1-\hat{f}_1^{\epsilon})
+
\left(f_2 - \hat{f}_2^{\epsilon}\right)
\right\|_{\infty}\\
=& \left\|
(I_{(\frac{1}{2},1]}\circ \sigma_{\operatorname{sigmoid}}\circ g)f_1-\hat{f}_1^{\epsilon}
\right\|
+
\left\|
f_2 - \hat{f}_2^{\epsilon}
\right\|_{\infty}\\
= &
\left|I_{(\frac{1}{2},1]}\circ \sigma_{\operatorname{sigmoid}}\circ g\right|
\left\|
f_1-\hat{f}_1^{\epsilon}
\right\|
+
\left\|
f_2 - \hat{f}_2^{\epsilon}
\right\|_{\infty}
\\ 
\leq &
1
\frac{\epsilon}{2}+
\frac{\epsilon}{2} = \epsilon.
\end{aligned}
$$
Hence, $\mathcal{V}\subseteq \overline{\PCNN{\fff}}$.  The conclude the proof.
\end{proof}
	
In what follows, we denote $d_{\infty}(f,g)\triangleq 
\left\|
f-g
\right\|$.
%
\begin{proof}[{Proof of Theorem~\ref{thrm_improvement_largeness}}]
Since $\fff$ is dense in $C(X,\rrD)$, there exists an $f \in \fff$ for which $f_n(x)>1$ for each $n=1,\dots,D$ and every $x \in X$.  Similarly, since $\ggg$ is dense in $C(X,\rr)$ then it must contain some non-constant $g$ for which $g(x)>0$ for all $x \in X$. Since $g$ is non-constant and $X$ is path-connected then the intermediate value theorem (\citep[Theorem 24.3]{munkres2014topology}) implies that there exist a curve $\gamma:(0,1)\rightarrow X$ such that 
\begin{equation}
    g\circ \gamma(t_1)< g\circ \gamma(t_2)
\label{eq_gap_1}
,
\end{equation}
for each $0<t_1<t_2<1$.  Since  $\sigma_{\operatorname{sigmoid}}$ is injective, then~\eqref{eq_gap_1} implies that
\begin{equation}
    0
        <
    \sigma_{\operatorname{sigmoid}}(g\circ \gamma(t_1))
        <
    \sigma_{\operatorname{sigmoid}}(g\circ \gamma(t_2))
        <
    1
\label{eq_inequality}
.
\end{equation}
Consider the subset $\mathcal{Z}\triangleq \left\{
(I_{(\sigma_{\operatorname{sigmoid}}(g\circ \gamma(t_1),1]}\circ \sigma_{\operatorname{sigmoid}}\circ g)f
\right\}\subseteq \PCNN{\fff}$. 

By~\eqref{eq_inequality}, for every $0<t_1<t_2<1$ we compute:
$$
\begin{aligned}
& d_{\infty}\left(
(I_{(\sigma_{\operatorname{sigmoid}}(g\circ \gamma(t_1),1]}\circ \sigma_{\operatorname{sigmoid}}\circ g)f,
(I_{(\sigma_{\operatorname{sigmoid}}(g\circ \gamma(t_2)),1]}\circ \sigma_{\operatorname{sigmoid}}\circ g)f
\right)
\\
= &
\sup_{x \in X}
\left\|
(I_{(\sigma_{\operatorname{sigmoid}}(g\circ \gamma(t_1)),1]}\circ \sigma_{\operatorname{sigmoid}}\circ g(x))f(x)
-
(I_{(\sigma_{\operatorname{sigmoid}}(g\circ \gamma(t_2)),1]}\circ \sigma_{\operatorname{sigmoid}}\circ g(x))f(x)
\right\|\\
= &
\sup_{t \in (0,1)}
\left\|
(I_{(\sigma_{\operatorname{sigmoid}}(g\circ \gamma(t_1)),1]}\circ \sigma_{\operatorname{sigmoid}}\circ g(
\gamma(t)
))f(x)
-
(I_{(\sigma_{\operatorname{sigmoid}}(g\circ \gamma(t_2)),1]}\circ \sigma_{\operatorname{sigmoid}}\circ g(
\gamma(t)
))f(x)
\right\|\\
= &
\left\|
(I_{(\sigma_{\operatorname{sigmoid}}(g\circ \gamma(t_1)),1]}\circ \sigma_{\operatorname{sigmoid}}\circ g(
\gamma(t_1)
))f\left(\gamma(t_1)\right)
-
(I_{(\sigma_{\operatorname{sigmoid}}(g\circ \gamma(t_2)),1]}\circ \sigma_{\operatorname{sigmoid}}\circ g(
\gamma(t_1)
))f\left(\gamma(t_1)\right)
\right\|\\
>& 1.
\end{aligned}
$$
Therefore, $\mathcal{Z}$ is an uncountable discrete subspace of $\overline{\PCNN{\fff}}$.  Hence, it cannot contain a countable dense subset (since each singleton in $\mathcal{Z}$ is an open subset of $\mathcal{Z}$).  In other words, $\mathcal{Z}$ is not separable.  Since every subset of a separable space it is itself separable, then we conclude that $\overline{\PCNN{\fff}}$ is not separable.  That is, (i) holds.  

Next, we show (ii).  Lastly, we show that $C(X,\rrD)$ is a subset of $\overline{\PCNN{\fff}}$.  Since $\ggg$ contains the zero-function $\zeta(x)=0$, for all $x \in X$, then for every $f \in \fff$, the function $\tilde{f}\triangleq \left(I_{(\frac{3}{4},1]}\circ \sigma_{\operatorname{sigmoid}}\circ \zeta\right)f$ belongs to $\PCNN{\fff}$.  However, by construction
$$
\left(I_{(\frac{3}{4},1]}\circ \sigma_{\operatorname{sigmoid}}\circ \zeta(x)\right) =1,
$$
for all $x \in X$.  Therefore, 
$
\tilde{f}\left(I_{(\frac{3}{4},1]}\circ \sigma_{\operatorname{sigmoid}}\circ \zeta\right)f = 1f =f.
$  
Hence, $\fff\subseteq \PCNN{\fff}$.  Since density is transitive, then it follows that $C(X,\rrd)\subseteq \overline{\PCNN{\fff}}$.  However, any function in $\mathcal{Z}$ does not belong to $C(X,\rrD)$ since each function in $\mathcal{Z}$ is discontinuous.  Thus, $C(X,\rrD)$ is a proper subset of $\overline{\PCNN{\fff}}$.  This concludes the proof.  
\end{proof}

\subsection{Proof of Theorem~\ref{thm_generic_sets}}\label{ss_Proofs_set_valued}
The proof of our main universal approximation Theorem (i.e., Theorem~\ref{thrm_UAT_PC}) relies on approximating the parts of a minimal representation of the target function $f\in PC(X,\rrD)$ using the decoupled upper-bound of Proposition~\ref{prop_decoupling_lemma}.  The $f_1,\dots,f_{N(f)}$ in any such minimal representation are approximated using a classical Universal Approximation Theorem, such as \cite{kidger2020universal}.  Next, the $K_1,\dots,K_{N(f)}$ is such a minimal representation are approximated using the following next Lemma (which quantitatively refines Theorem~\ref{thm_generic_sets}) guarantees that the \textit{deep zero sets} (defined in~\eqref{eq_deep_zero_sets_definition}) are dense in $\text{Comp}(X)$ with respect to the Hausdorff metric.  The following is a quantitative refinement of Theorem~\ref{thm_generic_sets}.  
\begin{thrm}[Deep Zero-Sets are Universal: Quantitative Version]\label{lem_generic_sets}
Let $\sigma$ satisfy Assumption~\ref{ass_KL}, $\emptyset\neq K_1,\dots,K_N\subseteq X$ be compact, $0<\epsilon$, and set $\gamma\triangleq \sigma_{\text{sigmoid}}^{-1}(2^{-1}\epsilon)$.  There is a FFNN $\hat{c}\in {NN}_{d,N}^{\sigma}$ of width at-most $d+N+2$ for which the deep zero-sets
$
\hat{K}_n\triangleq \left\{x\in X:\, I_{(\gamma,1]}\circ \sigma_{\operatorname{sigmoid}}\circ \hat{c}(x)_n\right\}
$
$$
\max_{n=1,\dots,N}\,
d_{H}(K_n,\hat{K}_n)
    \leq 
\epsilon.
$$
In particular, if $X=[0,1]^d$ then:
\begin{enumerate}[(i)]
    \item if $\sigma=\max\{0,x\}$, then $\hat{f}$ can be implemented by a FFNN of constant width $2d+10$ and depth $\mathcal{O}\left(
    (b\epsilon)^{-d+1}
    \right)$,
    \item if $\sigma \in C^{\infty}(\rr)$ then $\hat{f}$ can be implemented by a narrow FFNN with constant width $2d+2$ and depth $\mathcal{O}\left(b\epsilon^{-2d}\right)$,
\end{enumerate}
where in both cases (i) and (ii), $b>0$ is a (possibly different) constant which is independent of $\sigma$, $\epsilon$, and of $K_1,\dots,K_N$.
\end{thrm}
\begin{rremark}[{Implications of Lemma~\ref{lem_generic_sets} for Classification}]\label{lem_niceness_of_lemma}
Lemma~\ref{lem_generic_sets} strengthens the recent deterministic universal classification result in \citep[Theorem 3.11]{kratsios2020non} and the probabilistic classification results of \cite{farago1993strong}.  Moreover, the approximation does not omit sets of Lebesgue measure $0$, as do the results of \cite{caragea2020neural}.  
\end{rremark}
\begin{rremark}[{Discussion: Approximation Rates of Theorem~\ref{lem_generic_sets} always achieve the optimal depth-rates of \cite{pmlrv75yarotsky18a}}]\label{rremark_double_exponentiability}
In \cite{pmlrv75yarotsky18a}, it is shown that there are deep ReLU network approximating of constant width $2d+10$ and depth 
\begin{equation}
    \mathscr{O}([(\tilde{b}\epsilon)^{-d+1}])^{\boldsymbol{p}}
,
\label{eq_yaroski_comparison}
\end{equation}
the achieved the optimal approximate for any $\frac1{p}$-H\"{o}lder function, for $p \in [1,\infty)$; for some constant $\tilde{b}>0$.  
We make two observations between comparing the rates of~\eqref{eq_yaroski_comparison} and to our rates in Theorem~\ref{lem_generic_sets} (i); which describe a network of the same width whose depth grows at the rate:
\begin{equation}
    \mathcal{O}\left(
    (b\epsilon)^{-d+1}
    \right)
    \label{eq_ours_comparison}
    .
\end{equation}
First, we note that~\eqref{eq_ours_comparison} coincides with the optimal rates of~\eqref{eq_yaroski_comparison} in the case where $p=1$; that is when the target function is Lipschitz.  However, the "double exponential" dependence on $\epsilon$ and the regularity of the target function is always avoided from our problem; this is because the distance function $d(\cdot, K)$ can never be less regular than $p=1$.  

The case where $\sigma$ is smooth is similar.  Here, we notice that the rate (ii) achieves the rate of \cite{kratsios2021quantitative} only when approximating maximally regular target functions; i.e., Lipschitz functions.
\end{rremark}
\begin{proof}[{Proof of Lemma~\ref{lem_generic_sets}}]
Let $K_1,\dots,K_N\in \operatorname{Comp}(X)$, $N\in \mathbb{N}$, $0<\gamma<1$, and let $0<\epsilon$.  For each $n=1,\dots,N$, \citep[Lemma 1.2]{PallaschkePumplun2015DiscussionLipschitzMetricSpaces} guarantees that the map $C:\rrd\ni x \mapsto (d(\cdot,K_n))_{n=1}^N\in\rr^N$ is $1$-Lipschitz continuous.  
%
Therefore, as $\sigma$ is non-affine and continuously differentiable at at-least one point with non-zero derivative that point, then \citep[Corollary 42]{kratsios2021quantitative} applies.  Hence, there exists a FFNN $\hat{c}\in \NN[d,N][\sigma]$ of width at-most $d+N+2$ satisfying:
\begin{equation}
\max_{n=1,\dots,N}\max_{x \in X}
    \,
\left|d(x,K_n)-\hat{f}(x)_n\right|
\leq 
\max_{n=1,\dots,N}
\max_{x \in X}
    \,
\left|C(x)_n-\hat{c}(x)_n\right|< 2^{-1}\epsilon
\label{eq_proof_thrm_generic_sets_1}    
.
\end{equation}
Since $\sigma_{\text{sigmoid}}$ is continuous and monotone increasing then \citep[Theorem 1]{HoffmansResult2015ContinuityofInverseMonotoneFunction} implies that it is injective with continuous inverse on its image; thus, we can define $\gamma\triangleq \sigma^{-1}\left(2^{-1}\epsilon\right)$ and we define $\hat{K}_n\triangleq \left\{x\in X:\, I_{(\gamma,1]}\circ \sigma_{\text{sigmoid}}\circ \hat{c}(x)_n-1=0\right\}$.  Observe that, for each $n=1,\dots,N$, by~\eqref{eq_proof_thrm_generic_sets_1} we have that:
\begin{equation}
    \hat{K}_n= \left\{x\in X:\, \hat{c}(x)_n<2^{-1}\epsilon\right\}
    \label{eq_proof_thrm_generic_sets_2}    
    .
\end{equation}
It is enough to show the claim for an arbitrary $n=1,\dots,N$; thus, without loss of generality we do so for an arbitrary such $n$.  
Let us compute the Hausdorff distance between $K_n$ and $\hat{K}_n$.  If $x \in K_n$, then by~\eqref{eq_proof_thrm_generic_sets_1} and the definition of $d(\cdot,K_m)$ we compute:
\begin{equation}
\hat{f}(x)_n\leq 
    \left|\hat{c}_n(x)-0\right|
    =
    \left|\hat{c}_n(x)-d(x,K_n)\right| 
    \leq 2^{-1}\epsilon
    \label{eq_proof_thrm_generic_sets_3}    
    .
\end{equation}
Thus, $K_n\subseteq \hat{K}_n$; hence, $d(x,\hat{K}_n)=0$ whenever $x \in K_n$.  Whence, by the definition of the Hausdorff distance, we have:
\begin{equation}
    d_H(K_n,\hat{K}_n)=\max\left\{
    \sup_{x \in K_n} \, d(x,\hat{K}_n),\sup_{x\in \hat{K}_n}\, d(x,K_n)
    \right\}
    =
    \sup_{x\in \hat{K}_n}\, d(x,K_n)
    \label{eq_proof_thrm_generic_sets_0_bound_second_term}    
    .
\end{equation}
It remains to bound the right-hand side of~\eqref{eq_proof_thrm_generic_sets_0_bound_second_term}.  Let $x \in \hat{K}_n$.  By definition of $\hat{K}_n$ we have that $\hat{c}(x)_n\leq 2^{-1}\epsilon$.  Coupling this observation with the estimate~\eqref{eq_proof_thrm_generic_sets_1}, we find that for every $x \in \hat{K}$:
\begin{equation}
    d(x,K_n)
    \leq 2^{-1}\epsilon + |\hat{c}(x)_n| \leq 2^{-1}\epsilon + 2^{-1}\epsilon = \epsilon
    \label{eq_proof_thrm_generic_sets_4}    
    .
\end{equation}
Combining the estimate of~\eqref{eq_proof_thrm_generic_sets_4} with~\eqref{eq_proof_thrm_generic_sets_0_bound_second_term} yields the desired estimate:
$
d(K_n,\hat{K}_n)=
\sup_{x\in \hat{K}_n}\, d(x,K_n) \leq \epsilon
.
$
Lastly, since $C$ is $1$-Lipschitz then if $X=[0,1]^d$ then, the estimate in (i) follows form \citep[Theorem 2 (a) and (b)]{pmlrv75yarotsky18a} and the estimate (ii) holds by \citep[Corollary 42]{kratsios2021quantitative}.
\end{proof}
\subsection{Proof of Main Result - Theorem~\ref{thrm_UAT_PC}}
We may now return to the proof of Theorem~\ref{thrm_UAT_PC}.  
\begin{proof}[{Proof of Theorem~\ref{thrm_UAT_PC}}]
Let $f\in PC(X,\rrD)$ and let $\epsilon>0$.  By definition, $N(f)<\infty$ and therefore by Proposition~\ref{prop_existence} there exists some $((f_n,K_n))_{n=1}^{N(f)}\in [C(X,\rrD)\times \text{Comp}(X)]^{N(f)}$ with $R(((f_n,K_n))_{n=1}^{N(f)})=f$.  

Since $\sigma\in C(\rr)$ satisfies the Kidger-Lyons conditions, then Lemma~\ref{lem_generic_sets} implies that there exists some $\hat{c}\in \NN[d,N(f)]$ of width at-most $d+N(f)+2$ for which, the associated the deep zero-sets
$
\hat{K}_n\triangleq \left\{x\in X:\, I_{(\gamma,1]}\circ \sigma_{\operatorname{sigmoid}}\circ \hat{c}(x)_n\right\}
$ satisfy:
\begin{equation}
    d_{H}(K_n,\hat{K}_n)<(2N(f))^{-1}\epsilon,
    \label{eq_proof_thrm_UAT_PC_sets_estimate}
\end{equation}
for each $n=1,\dots,N(f)$.  Since $\sigma$ satisfies the Kidger-Lyons condition, then the universal approximation Theorem \citep[Corollary 42]{kratsios2021quantitative} implies that there are FFNNs $\hat{f}_1,\dots,\hat{f}_{N(f)}\in \NN[d,D]$ of width at-most $d+D+2$ satisfying the estimate:
\begin{equation}
    \max_{n=1,\dots,N(f)}\,\|f_n-\hat{f}_n\|_{\infty}<(2N(f))^{-1}\epsilon
    .
    \label{eq_proof_thrm_UAT_PC_functions}
\end{equation}
Combining the estimates~\eqref{eq_proof_thrm_UAT_PC_sets_estimate} and~\eqref{eq_proof_thrm_UAT_PC_functions} with the estimate of Proposition~\ref{prop_decoupling_lemma} yields:
\begin{align}
    \allowdisplaybreaks
    \nonumber
    D_{PC}\left(f\middle|\sum_{n=1}^{N(f)} \hat{f}_nI_{\hat{K}_n}\right) 
& \leq 
\sum_{n=1}^{N(f)}\, \|f_n-\hat{f}_n\|_{\infty} + \sum_{n=1}^{N(f)} \, d_H(K_n,\hat{K}_n)
\\
& \leq N(f)(2N(f))^{-1}\epsilon + N(f)(2N(f))^{-1}\epsilon \\
& = \epsilon
.
\end{align}
Thus, the result follows.  
\end{proof}
\subsection{Proof of Theorem~\ref{thrm_partition_learner}}\label{ss_App_proof_optimal_partitioning}	
Theorem~\ref{thrm_partition_learner} is a special case of the following, more detailed, result.  
\begin{thrm}[Existence of an Optimal Partition: Extended Version]\label{thrm_partition_learner_Extended_Version}
	Let $L:\rrD\rightarrow \rr$ be a continuous loss-function for which $L(0)=0$ and let $\{\hat{f}_n\}_{n=1}^N\subseteq C\left(\kkk,\rrD\right)$.  There exists a Borel measurable function $C:\rrD\rightarrow \{n\leq N\}$ such that for each $x \in \rrD$
	\begin{equation}
	\min_{m\leq N} L(\hat{f}_m(x)) 
	=
	L\left(
	\sum_{n\leq N} I_{n}(C(x)) \hat{f}_n^{\star}(x)
	\right)
	\label{eq_measurable_selector_definition_1}
	,
	\end{equation}
	where $I_n:{n=1}^N\rightarrow {0,1}$ and $I_n(m)=1$ if and only if $n=m$.  Moreover, for each Borel probability measure $\pp$ on $\rrd$ and each $\delta \in (0,1]$, there exists a compact subset $\kkk_{\delta,\pp}\subseteq \rrd$ such that:
	\begin{enumerate}[(i)]
		\item $\pp\left(\kkk_{\delta,\pp}\right)\geq 1-\delta$,
		\item $C$ is continuous on $\kkk_{\delta,\pp}
		,
		$
		\item For $n\leq N$, $K_n^{\star}\triangleq \kkk_{\delta,\pp} \cap C^{-1}[\{n\}]$ is a compact subset of $\kkk_{\delta,\pp}$ on which 
		\begin{equation}
		\min_{m\leq N} L(\hat{f}_m^{\star}(x)) = L(\hat{f}_n^{\star}(x))
		\label{eq_thrm_partition_learner_Extended_Version_optimality_property}
		,
		\end{equation}
		holds for each $x \in K_n^{\star}$.  
		\item $\kkk_{\delta,\pp} = \bigcup_{n\leq N} K_n^{\star}$ and for each $n,m\leq N$ if $n\neq m$ then $K_m^{\star}\cap K_n^{\star}=\emptyset$.  
		\item For each $n\leq N$, $K_n^{\star}$ is an open subset of $\kkk_{\delta,\pp}$ (for its relative topology).  In particular, $\kkk_{\delta,\pp}$ is disconnected.  
	\end{enumerate}
	Moreover, if $\xx\subseteq \kkk_{\delta,\pp}$ and $L(\hat{f}_n(x))< \min_{n\neq \tilde{n},\,\tilde{n}\leq N}L(\hat{f}_{\tilde{n}}(x))$ for each $n\leq N$ and each $x \in \xx_n$ then for each $n\leq N$
	\begin{itemize}
		\item $\xx_n\subseteq K_n^{\star}$,
		\item $\xx_n \cap K_m^{\star}= \emptyset $ for every $m\leq N$ distinct from $n$.   
	\end{itemize} 
\end{thrm}
\begin{proof}[{Proof of Theorem~\ref{thrm_partition_learner_Extended_Version}}]
	We begin  by establishing the existence of a measurable selector, that is, a measurable function $C$ satisfying (ii) on all of $\kkk$.  First, notice that the map $(n_0,x)\mapsto L(\hat{f}_{n_0}(x))$, from $\{n\leq N\}\times \rrd$ to $\rr$ can be represented by
	\begin{equation}
	L(\hat{f}_{n_0}(x)) 
	= 
	\sum_{n_0\leq N}I_{n_0}(n) L\left(\hat{f}_{n_0}(x)\right)
	\label{eq_redux_expression_1}
	,
	\end{equation}
	where $I_{n_0}(n)=1$ if $n=n_0$ and $0$ otherwise.  Next, observe that the map $x \mapsto \hat{f}_{n_0}(x)$ is continuous and this holds for each $n_0\leq N$.  Next, since $L$ is continuous, each $\hat{f}_{n_0}$ is continuous, and the composition of continuous functions is again continuous, then $L\circ \hat{f}_{n_0}$ is continuous.  Therefore, for each $n_0\leq N$ the map $x \mapsto L\left(\hat{f}_{n_0}(x)\right)$ is continuous from $\rrd$ to $\rr$.  Next, observe that for any fixed $x \in \rrd$, the map $n\mapsto \sum_{n_0\leq N} I_{n_0}(n)L\circ \hat{f}_{n_0}(x)$ is a simple-function, thus, it is Borel-measurable.  Hence, the map $(n,x)\mapsto \sum_{n_0\leq N}I_{n_0}(n) L\left(\hat{f}_{n_0}(x)\right)$ is a Carath\'{e}odory function (\citep[Definition 4.50]{InfiniteHitchhiker2006}).  
	
	Next, observe that the $\Phi$ multi-function taking any $x \in \mathbb{R}^d$ to the set $\Phi(x)\triangleq \{n\leq N\}$ is constant.  Therefore, it is a weakly-measurable correspondence in the sense of \citep[Definition 18.1]{InfiniteHitchhiker2006}.  Moreover, since the set $\{n\leq N\}$ is a non-empty finite set then $\Phi(x)=\{n\leq N\}$ is non-empty and compact for each $x \in \rrd$.  Hence, $\Phi$ is a weakly-measurable correspondence taking non-empty and compact-values in the separable metric space $\rrd$.  Therefore, the conditions for the \citep[Measurable Maximum Theorem; Theorem 18.19]{InfiniteHitchhiker2006} are met and thus, there exists some Borel-measurable function $C:\rrd\rightarrow \{n\leq N\}$ satisfying~\eqref{eq_measurable_selector_definition_1}
	\begin{equation}
	\min_{n\leq N}
	\sum_{n_0\leq N}I_{n_0}(n) L\left(\hat{f}_{n_0}(x)\right)
	= 
	L\left(\hat{f}_{C(x)}(x)\right)
	\label{eq_existence_measurable_selector}
	,
	\end{equation}
	for every $x \in \rrd$.  Since $L(0)=0$, then, combining~\eqref{eq_redux_expression_1} and~\eqref{eq_existence_measurable_selector} we find that for every $x \in \rrd$ the following holds~\eqref{eq_measurable_selector_definition_1} holds.	 This establishes the first claim.  
	
	Since $\pp$ is a Borel probability measure on $\kkk$ and since $\kkk$ is separable and metrizable then by \citep[Theorem 13.6]{klenke2013probability} $\pp$ is a regular Borel measure on $\rrd$.  Since $\rrd$ is a metric space, it is a second-countable Hausdorff, and since $\rrd$ is also locally-compact (since each $x \in \rrd$ satisfies $x \in \{z\in \rrd:\|x-z\|<1\}$ and the latter is compact by the Heine-Borel theorem) then $\rrd$ is a locally-compact Hausdorff space; thus, \citep[Theorem 7.8]{folland2013real} applies therefore $\pp$ is a Radon measure on $\rrd$.  	
	Since $\pp$ is a regular Borel measure on $\rrd$, $\rrd$ is a locally-compact Hausdorff space, $C$ is a Borel measurable function to $\{n\leq N\}$, and $\{n\leq N\}$ is a separable metric space (when viewed as a metric sub-space of $\rrflex{N}$), and $\pp(\rrd)=1<\infty$ since $\pp$ is a probability measure, then \citep[Lusin's Theorem; Exercise 2.3.5]{FedererGeometricMeasureTheory1969} applies; whence, for every $\delta\in (0,1]$ there exists some non-empty compact subset $\kkk_{\delta,\pp}\subseteq \rrd$ for which $C|_{\kkk_{\delta,\pp}}\in C(\kkk_{\delta,\pp},\{n\leq N\})$ is continuous and for which $\pp(\kkk_{\delta,\pp})\geq 1-\delta$.  This establishes (i) and (ii).  
	
	Since $C$ is continuous on $\kkk_{\delta,\pp}$, each $\{n\}$ is closed in $\{n_0\leq N\}$, since the pre-image of closed sets under continuous functions is again closed.  Therefore, for each $n\leq N$, $C^{-1}[\{n\}]$ is a closed subset of $\rrd$.  Since each $K_n^{\star}=C^{-1}[{n}]\cap \kkk_{\delta,\pp}$ is a closed subset of $\kkk_{\delta,\pp}$ and since $\kkk_{\delta,\pp}$ is compact then each $K_n^{\star}$ is compact.  For each $n\leq N$, $x \in K_n^{\star}$ only if $x \in C^{-1}[\{n\}]$, only if $C(x)=n$ and therefore condition:
	$$
	L(\hat{f}(x))=\min_{m\leq N}L(\hat{f}_{m}^{\star}(x)),
	$$
	holds.  This gives~\eqref{eq_thrm_partition_learner_Extended_Version_optimality_property}.  Observe that
	$$
	\kkk_{\delta,\pp} \cap \rrd = \kkk_{\delta,\pp} \cap C^{-1}\left[\{n\leq N\}\right] 
	=
	\kkk_{\delta,\pp} \cap \bigcup_{n\leq N} C^{-1}\left[\{n\}\right]
	=
	\bigcup_{n\leq N} \kkk_{\delta,\pp} \cap  C^{-1}\left[\{n\}\right]
	= \bigcup_{n\leq N}K_{n}^{\star}.
	$$
	Therefore, (iii) holds.  
	
	Lastly, let $n,m\leq N$ and $n\neq m$.  Suppose that there exists some $x \in K_n^{\star}\cap K_m^{\star}$.  Since $x \in K_n^{\star}$ then $C(x)=n$ and since $x \in K_m$ then $C(x)=m$; therefore $n=C(x)=m$ and $n\neq m$, a contradiction.  Hence, $K_n^{\star}\cap K_m^{\star}=\emptyset$ for each $n,m\leq N$ with $n\neq m$ and therefore (iv) holds.  
	
	Since $\{n\leq N\}$ is a discrete metric space then for each $n\leq N$, the singleton set $\{n\}$ is open in $\{n\leq N\}$.  Therefore, by (ii), $C$ is continuous on $\kkk_{\delta,\pp}$ (for the relative topology) hence $C^{-1}[\{n\}]\cap \kkk_{\delta,\pp}$ is open in $\kkk_{\delta,\pp}$.  By (iv), pick $1<n\leq N$, since $K_n^{\star}\cap K_1^{\star}=\emptyset$ then \citep[Definition 3.23]{munkres2014topology} is satisfies and therefore $\kkk_{\delta,\pp}$ is not connected, i.e$.:$ it is disconnected.  This gives (v).  
	
	If $\xx\subseteq \kkk_{\delta,\pp}$ then by (iv) $\xx\subseteq \bigcup_{n\leq N} K_n^{\star}$.  Since, for each $n\leq N$ and each $x \in \xx_n$, $L(\hat{f}_n(x))<\min_{\tilde{n}\neq n,\, \tilde{n}\leq N}L(\hat{f}_{\tilde{n}}(x))$ then $\xx_n\subseteq K_{n}^{\star}$ by (iii).  Note, moreover, that by (iv) $K_n^{\star}\cap K_m^{\star}=\emptyset$ for every $n,m\leq N$ with $n\leq m$ then there exists a unique $n^x \in \{n\leq N\}$ such that $x \in K^{\star}_{n^{x}}$ and therefore $\xx_n\cap K_{m}^{\star} = \emptyset $ for each $m\leq N$ with $m\neq n$.  
\end{proof}
%
		%
		%
		%
		%
		%
		%
\subsection{Proof of {Proposition~\ref{cor_dat_driven_opt_alloc}}}\label{ss_proof_prop_}
\begin{proof}[{Proof of Proposition~\ref{cor_dat_driven_opt_alloc}}]
	By construction, for each $n\leq N$, $K_n^{\xx}\neq \emptyset$ and $K_n^{\xx}\cap \xx_n\subseteq \emptyset$.  Therefore, there exists some $\tilde{x} \in K_n^{\xx}\cap \xx_n$.  Since the set $\{(x,y)\in \xx^2:\,x\neq y\}$ is a non-empty finite set then the minimum $\min_{x,y\in \xx,\, x\neq y} \, \|x-y\|$ is attained and it is non-zero since $\|x-y\|>0$ for each $x,y \in \xx$ with $x\neq y$.  Thus, $\Delta(\xx)>0$.  Therefore, $\operatorname{Ball}\left(\tilde{x},\Delta(\xx)\right)\neq \emptyset$ and contained in $K_n^{\xx}$.  Therefore, by definition of the interior of $K_n^{\xx}$ we compute
	$$
	\intt[K_n^{\xx}] = \bigcup_{
		\underset{U \mbox{ is open}}{K_n^{\xx}\subseteq U\subseteq \rrd}
	} U 
	\supseteq
	\operatorname{Ball}\left(\tilde{x},\Delta(\xx)\right)\neq \emptyset
	.
	$$
	This gives (iii).  
	
	For (iv), note that by definition of $\Delta(\xx)$, $\operatorname{Ball}(x,\Delta(\xx))\cap \operatorname{Ball}(y,\Delta(\xx))=\emptyset$ for each $x,y\in \xx$ with $x\neq y$ and by construction
	\begin{equation}
	K_n^{\xx} = \bigcup_{x \in \xx_n} \operatorname{Ball}(x,\Delta(\xx))
	\label{eq_little_identitiy}
	.
	\end{equation}
	Therefore, for $n,m\leq N$ with $n\neq m$, by construction $\xx_n \cap \xx_m=\emptyset$ and therefore~\eqref{eq_little_identitiy} implies that $K_n^{\xx} \cap K_m^{\xx}=\emptyset$.  
	
	
	For (i), note that, if $f$ is of the form~\eqref{eq_sub_patterns}, then $f= \sum_{n\leq N} f_n I_{\intt[K_n]}$.  Therefore, for $x_1,x_2 \in \xx$, we have that there exits some $n\leq N$ for which $f(x_i)=f_n(x_i)$ for $i=1,2$ if (but not only if) $x_1,x_2 \in K_n^{\xx}$.  Next, observe that by construction $K_n^{\xx} \cap \xx = \xx_n$.  Therefore, 
	\begin{equation}
	\pp\left(
	(\exists n \leq N) \, f(x_i)=f_n(x_i)\mbox{ for } i=1,2
	\right) \geq 
	\pp\left(
	(\exists n \leq N) \, x_i \in K_n \mbox{ for } i=1,2
	\right)
	\label{eq_first_approx}
	.
	\end{equation}
	The central results of \cite{Algo0extensionproblemYuvalCalinescuHoward2004} implies that
	\begin{equation}
	\pp\left(
	(\not\exists n \leq N)
	(\exists n \leq N) \, x_i \in K_n \mbox{ for } i=1,2
	\right)
	\leq \frac{2^3}{\bar{\Delta}(\xx)} \|x_1-x_2\| \left(
	\sum_{i= \# \operatorname{Ball}_{\xx}\left(x_1,2^{-3}\bar{\Delta}(\xx)\right)}^{
		\# \operatorname{Ball}_{\xx}\left(x_1,\bar{\Delta}(\xx)\right)
	} i^{-1}
	\right)
	\label{eq_first_approx_Yuval_et_al_bounds}
	.
	\end{equation}
	For each positive integer $I$, since the partial sums of the harmonic series $\sum_{i=1}^I i^{-1}$ are bounded above by $\ln(I) +1$, then~\eqref{eq_first_approx_Yuval_et_al_bounds} implies that
	\begin{equation}
	\begin{aligned}
	\pp\left(
	(\not\exists n \leq N)
	(\exists n \leq N) \, x_i \in K_n \mbox{ for } i=1,2
	\right)
	\leq & \frac{2^3}{\bar{\Delta}(\xx)} \|x_1-x_2\| \left(
	\sum_{i= \# \operatorname{Ball}_{\xx}\left(x_1,2^{-3}\bar{\Delta}(\xx)\right)}^{
		\# \operatorname{Ball}_{\xx}\left(x_1,\bar{\Delta}(\xx)\right)
	} i^{-1}
	\right)
	\\
	\leq &
	\frac{2^3}{\bar{\Delta}(\xx)} \|x_1-x_2\| \left(
	\sum_{i= 1}^{
		\# \operatorname{Ball}_{\xx}\left(x_1,\bar{\Delta}(\xx)\right)
	} i^{-1}
	\right)
	\\
	\leq &
	\frac{2^3}{\bar{\Delta}(\xx)} \|x_1-x_2\| \left(
	\sum_{i= 1}^{
		\# \xx
	} i^{-1}
	\right)
	\\
	\leq &
	\frac{2^3}{\bar{\Delta}(\xx)} \|x_1-x_2\| \left(
	\ln\left(\# \xx \right) + 1
	\right)
	.
	\end{aligned}
	\label{eq_first_approx_Yuval_et_al_bounds_refined}
	\end{equation}
	Combining~\eqref{eq_first_approx} and~\eqref{eq_first_approx_Yuval_et_al_bounds_refined} we obtain the following bound
	\begin{equation}
	\pp\left(
	(\exists n \leq N) \, f(x_i)=f_n(x_i)\mbox{ for } i=1,2
	\right) \geq 
	1 - \frac{2^3}{\bar{\Delta}(\xx)} \|x_1-x_2\| \left(
	\ln\left(\# \xx \right) + 1
	\right)
	\label{eq_iv_approx_bound}
	.
	\end{equation}

	Lastly, for (ii), by \cite{Bartalmetricapprox} and the definition of Subroutine~\ref{subroutine_GET_PARTITION} all but the last step of Subroutine~\ref{subroutine_GET_PARTITION} runs in polynomial time.  The final step of Subroutine~\ref{subroutine_GET_PARTITION} runs in $\mathscr{O}(N\log(N))$ since it is simply sorting procedure.  Therefore Subroutine~\ref{subroutine_GET_PARTITION} terminates in polynomial time.  
\end{proof}

\section*{Funding}
This work was supported by the ETH Z\"{u}rich Foundation.  Anastasis Kratsios was also supported by the European Research Council (ERC) Starting Grant 852821---SWING.

\bibliographystyle{elsarticle-num}
\bibliography{main}
\end{document}